\documentclass[twoside,11pt]{article}

\usepackage{jmlr2e}

\usepackage{graphics} 
\usepackage{epsfig} 
\usepackage{mathptmx} 
\usepackage{times} 
\usepackage{amsmath} 
\usepackage{amssymb}  
\usepackage{circuitikz}
\usepackage{pgfplots}
\pgfplotsset{compat=1.5}
\usepackage{float}
\usepackage{mathtools}
\usepackage{enumerate}
\usepackage{algcompatible}
\usepackage[section]{algorithm}
\usepackage[mathscr]{euscript}
\usepackage{bm}
\usepackage{cite}
\usepackage{caption}
\usepackage{subcaption}
\usepackage{algorithmicx}
\usepackage{algpseudocode}
\usepackage[export]{adjustbox}

\newcommand{\R}{\ensuremath{\mathbb{R}} }

\newcommand{\n}{\tt{n}}

\renewcommand{\i}{\tt{i}}

\renewcommand{\thesection}{\arabic{section}}

\newcommand{\bs}[1]{\boldsymbol{#1}}
\renewcommand{\thesection}{\arabic{section}}

\newtheorem{theo}{Theorem}[section]
\newtheorem{problem}[theo]{\bfseries Problem}

\begin{document}

\title{Graphical estimation of multivariate count time series}

\author{\name Vurukonda Sathish \email sathish@ee.iitb.ac.in \\
       \addr Department of Electrical Engineering\\
       Indian Institute of Technology Bombay, \\Powai, 400076, India
       \AND
       \name Debraj Chakraborty \email dc@ee.iitb.ac.in \\
      \addr Department of Electrical Engineering\\
      Indian Institute of Technology Bombay, \\Powai, 400076, India
   		\AND
   		\name Siuli Mukhopadhyay \email siuli@math.iitb.ac.in \\
   		\addr Department of Mathematics\\
   		Indian Institute of Technology Bombay, \\Powai, 400076, India}
\editor{--}


\maketitle

\begin{abstract}
The problems of selecting partial correlation and causality graphs for count data are considered. A parameter driven generalized linear model is used to describe the observed multivariate time series of counts. Partial correlation and causality graphs corresponding to this model explain the dependencies between each time series of the multivariate count data. In order to estimate these graphs with tunable sparsity, an appropriate likelihood function maximization is regularized with an $\ell_{1}$-type constraint. A novel MCEM algorithm is proposed to iteratively solve this regularized MLE. Asymptotic convergence results are proved for the sequence generated by the proposed MCEM algorithm with $\ell_{1}$-type regularization. The algorithm is first successfully tested on simulated data. Thereafter, it is applied to observed weekly dengue disease counts from each ward of Greater Mumbai city. The interdependence of various wards in the proliferation of the disease is characterized by the edges of the inferred partial correlation graph. On the other hand, the relative roles of various wards as sources and sinks of dengue spread is quantified by the number and weights of the directed edges originating from and incident upon each ward. From these estimated graphs, it is observed that some special wards act as epicentres of dengue spread even though their disease counts are relatively low.
\end{abstract}



\section{Introduction}

The dependencies between multiple interdependent time series data, such as the temperatures recorded from neighbouring geographical regions \citep{Appgeo}, stock prices from related markets \citep{Appstock}, biological signals that measure activities in the human brain \citep{Appbio} etc, are studied using graphical models. In a graphical model, each data series is represented as a vertex of a graph and the dependencies between the different time series are represented by edges in the graph. In many applications such as finance, insurance, biomedical, public health, etc, time series data is frequently measured in the form of counts \citep{jalalpour,freeland}. In this paper, we propose a novel algorithm for inferring the dependencies and the strength of the dependencies (i.e., the graphical model) for multiple inter-related {\it{count}} time series data. The algorithm is applied to discover the spreading pattern of dengue infections in a large Indian city. 

In graphical models, the dependencies between multiple time series are represented either by undirected or directed graphs (e.g., see \citealp{Granger,Eichler201222,DahlhausEichler,7091904} and references therein). First we review the undirected graphical models in brief and thereafter review the directed graphical models. Graphical models were introduced initially for multiple random variables in \citep{Dempster72} using the inverse covariance matrix of the associated random vectors. \citep{Brillinger96remarksconcerning} introduced such models for multivariate time series, where the dependencies between the nodes were represented using partial correlations. In \citep{Songsiri2009} and \citep{Graphical-Models3}, graphical description of vector autoregressive (AR) models were considered for stationary time series and defined in terms of the inverse spectral density matrix. These results were extended to graphical models with vector autoregressive moving average (ARMA) models in \citep{ARMAgraphical}. In \citep{Songsiri:2010}, a method to control the sparsity of the estimated graph was introduced through an $\ell_{1}$-type regularization. Such a regularization (often called Least Absolute Shrinkage and Selection Operator (LASSO)) on the standard maximum likelihood estimation (MLE) method, reduces overfitting and provides control over the sparsity of the estimated graphical model \citep{LASSO}. 

Directed graphical models were introduced in \citep{wright1921correlation,wright1934method} as path diagrams for the study of hereditary properties, linking parents and children graphically by arrows. For multivariate time series, one series is said to be causal for another series if the prediction of second series using all available information except the first series, can be improved by adding the available information about the first series. This definition of the so-called Granger causality was introduced in \citep{Granger}. Using this notion, the causalities between multiple time series were represented as directed Granger causality graphs, 
where each time series was represented as a vertex and the causality between two different time series was represented by a directed edge \citep{sims1972money,pierce1977causality}. 
The multivariate autoregressive (AR) model of second order stationary time series was used to obtain Granger causality graph \citep{Tjstheim,HSIAO1982} where the causalities between multiple time series were calculated directly from the autoregressive coefficients. The Granger causality graphs were studied by \citep{EICHLER2006,EICHLER200733422} through multivariate autoregressive (AR) models of second order stationary time series. Later, sparse Granger causality graphs were discussed in \citep{songsiri2013sparse,songsiri2015learning,zorzi2016,zorzi2018,zorzi2019} through multivariate AR time series.

In all the above papers, the graphical models were discussed for time series of Gaussian data. In contrast to the graphical models for Gaussian data, literature on graphical models for count data is still sparse. The undirected graphical models were introduced for multivariate count data in \citep{park2019high} using a Bayesian approach to inference. Later, the undirected graphical models were considered for multivariate count data in \citep{roy2020nonparametric} to obtain the complex interactions between genes using a pseudo-likelihood based algorithm. Unlike these directed models, Poisson directed acyclic graphical (DAG) models, also referred to as Bayesian Networks \citep{park2015learning,park2017learning}, were obtained for multivariate count data in \citep{hue2021structure}. In all these papers, the graphical models are obtained for multiple random variables. In contrast, our aim in this paper is to estimate undirected and directed graphical model for multivariate count time series where the dependencies between these time series are represented using causalities and partial correlations.

To model count time series, generalized linear models (GLM) were discussed in \citep{Zeger,Li1994,fahrmeir,fahrwagen,durbin2000}. In \citep{Bayesian:1}, parameter driven models were used for modeling multivariate time series of counts, using a latent process to account for the correlation between the observed counts. In this paper, we choose this model for our study. In this model, the latent process is unknown. While the unknown parameters in such models can, in principle, be estimated using the well known expectation and maximization (EM) algorithm \citep{emalgo}, associated difficulties in calculating the conditional density functions requires us to use the the Monte Carlo expectation and maximization (MCEM) algorithm \citep{doi:10.1080/01621459.1990.10474930,10.2307/2291149}. In this paper, we extend a combination of the methods from \citep{emalgo,doi:10.1080/01621459.1990.10474930,10.2307/2291149} to infer graphical models with tunable sparsity, using LASSO.

First we introduce certain choices about the model, formulation and methodology. 
Following \citep{Granger,Songsiri:2010}, the dependencies between the observed multivariate time series of counts are represented using partial correlations and causalities. However partial correlations and causalities between multiple time series can only be calculated easily for second order stationary time series. In our case, the observed count data is usually non stationary, e.g., the data might be seasonal in nature or follow various trends \citep{glm_book}. To address this issue, we introduce an unobserved multivariate second order stationary latent process to calculate the partial correlation and causality between the observed multiple count time series. We define the partial correlation and causality between elements of the observed time series in terms of the partial correlation and causality between elements of the stationary latent processes. As in \citep{10.2307/2291149}, the observed count data given the latent process, follows a Poisson distribution, while the multivariate latent process follows an autoregressive (AR) Gaussian model. 


It is known that the partial correlation between multiple time series can be computed from the inverse spectral density matrix of a multivariate stationary time series \citep{Time-Series}. Thus the partial correlations for the stationary latent process in our model are obtained in terms of the parameters of the latent AR model \citep{Brockwell} using the inverse spectral density matrix. Further, to obtain the partial correlations and causalities for the observed multivariable count data from the stationary latent process, we need to estimate unknown parameters of the AR model. Since the number of possible edges in the partial correlation graph can be large, overfitting can be a potential issue in this estimation problem \citep{overfitting}. Thus a maximum likelihood estimation (MLE), with $\ell_{1}$-type regularization term involving the partial correlation constraints, is formulated to get the desired sparsity in the estimated partial correlation and causality graph. 
Since the latent process is unknown, a direct maximization of the log-likelihood function with $\ell_{1}$-type regularization is not possible. Following \citep{10.2307/2291149}, an MCEM algorithm with $\ell_{1}$-type regularization is proposed to estimate the model parameters. In this algorithm, samples from the conditional density function of the latent process given observed data, are generated using the Metropolis-Hastings algorithm \citep{10.2307/2291149}.

First, the proposed MCEM algorithm with $\ell_{1}$-type regularization is tested on a simulated multivariable count dataset. The regularization parameter is estimated using the trade-off curve \citep{Jitkomut} between log-likelihood and $\ell_{1}$-type regularization function. We collect all the graphs along the trade-off curve for different values of regularization parameter and then choose the partial correlation and causality graph based on the Bayes information criterion (BIC)\citep{Jitkomut}. 

Next, we use the developed algorithm to estimate the partial correlation and causality graphs for the observed dengue disease data. The dengue data under investigation consist of weekly dengue counts in a six year period, January $2010$ to December $2015$, collected from each ward of Greater Mumbai city, India. From these estimated partial correlation and causality graphs, we infer the number of undirected, incoming and outgoing edges for each ward. These edges and their weights give a quantitative measure of the interdependence  and directionality of the disease prevalence and development in the various wards. Surprisingly, we observe that some wards act as the epicentres of disease spread even though their absolute disease counts are relatively low.


This manuscript is an extended version of \citep{sathish2019topology} which was presented in IEEE 58th Conference on Decision and Control (CDC), December 2019. It differs from \citep{sathish2019topology} as follows:
\begin{enumerate}
	\item In this manuscript, in addition to the problem of estimating partial correlation graphs studied in \citep{sathish2019topology}, the problem of estimating causality graphs is also considered. Causality graphs were not studied in \citep{sathish2019topology}.
	\item No real world application was presented in \citep{sathish2019topology}. Here, the developed algorithms are applied on observed weekly dengue disease counts in each of the 24 wards of Greater Mumbai city from 2010- 2015, to learn directed and undirected models for the spread of dengue disease.
\end{enumerate}

Our contributions in this paper are as follows.
\begin{enumerate}
	\item We formulate an optimization problem of maximum likelihood estimation (MLE) of observed multivariate count data with $\ell_{1}$-type regularization to estimate the partial correlation and causality graphs.
	\item A Monte Carlo expectation and maximization (MCEM) algorithm with $\ell_{1}$-type regularization on multivariate count data is proposed to solve the formulated problem.
	\item The proposed MCEM algorithm with $\ell_{1}$-type regularization is tested successfully using simulated data. 
	\item The partial correlation and causality graphs are estimated for the observed weekly dengue disease counts from each ward of Greater Mumbai city. 
\end{enumerate}


\section{Preliminaries and problem formulation}\label{sec:Notation}
\subsection{Parameter driven model}\label{sec:Parameter}
Let \{$\mathbf{Y}(t)\in \mathbb{R}^{n}, t = 1\dots N\}$ be an observed multivariate time series of counts. An unobserved multivariate second order stationary latent process \{$\mathbf{X}(t)\in \mathbb{R}^{n}, t = 1\dots N\}$ is used to introduce the correlation between observations measured at successive time points and individual time series. To model $\mathbf{Y}(t)$ using $\mathbf{X}(t)$, a parameter driven model introduced in \citep{Bayesian:1} is considered. Let $ \textbf{z}_{t,i} $ be a $ q \times 1 $ vector of covariates at time $ t $ for the $ i $th time series and $ \bs{\beta}_{i} $ be a $ q \times 1 $ vector of regression coefficients corresponding to the $ i $th covariate vector $ \textbf{z}_{t,i} $. Covariate vector $ \textbf{z}_{t,i} $ is a function of time $ t $ and is usually used to model the explanatory variables such as trend, seasonality etc. In this model, the counts $ Y_{i}(t) $, given the latent process $ X_{i}(t) $ and the covariates $ \textbf{z}_{t,i} $, follow Poisson distribution with conditional mean  $\mu_{i}(t) = \text{exp}(\textbf{z}_{t,i}^{T}\bs{\beta}_{i} + X_{i}(t)) $, which is denoted by 
\begin{align}\label{A1}
Y_{i}(t)|X_{i}(t),\textbf{z}_{t,i}\sim \text{Poisson}(\mu_{i}(t)).
\end{align}
We assume the latent process $ X_{i}(t) $ in the model (\ref{A1}) follows an AR(p) process given by
\begin{align}\label{A2}
X_{i}(t) = \sum_{j = 1}^{n}a_{ij}(1)X_{j}(t-1)+\dots+\sum_{j = 1}^{n}a_{ij}(p)X_{j}(t-p)+ \epsilon_{i}(t),
\end{align} 
where $ a_{ij}(.) $'s are autoregressive coefficients and $\epsilon_{i}(t)$ is a Gaussian white noise having zero mean and variance $\sigma_{i}^{2}$.
Let $A_{k}$ = \{$ a_{ij}(k) $\} be an $ n\times n $ matrix of autoregressive coefficients and $\mathbf{\epsilon}(t)$ := ($\epsilon_{1}(t),\dots,\epsilon_{n}(t)$) be iid Gaussian random variables with covariance matrix $ \Sigma $ := diag($ \sigma_{1}^{2},\dots,\sigma_{n}^{2} $). For $k=1,\dots,p$, the vector notation of the AR(p) multivariate latent process is given by
\begin{align}\label{A3}
\mathbf{X}(t) = \sum_{k = 1}^{p}A_{k}\mathbf{X}(t-k) + \mathbf{\epsilon}(t),\qquad \mathbf{\epsilon}(t)\sim N(\mathbf{0},\Sigma).
\end{align}    
Let $\bs{\beta}\coloneqq \begin{bmatrix}\bs{\beta}_{1}&\hdots &\bs{\beta}_{n}\end{bmatrix}\in \mathbb{R}^{q\times n}$ and $\bs{\sigma}\coloneqq \begin{bmatrix}\sigma_{1}&\hdots &\sigma_{n}\end{bmatrix}^{T}\in \mathbb{R}^{n\times 1}$. Also assume that the parameter space is denoted by $\Omega$. Define the operator vec($\cdot$), which converts the matrix into a vector by setting columns of the matrix as a vector. Then, $\bs{\theta}\coloneqq \text{vec}\Big(\begin{bmatrix}\bs{\beta}^{T}&A_{1}&\hdots&A_{p}&\bs{\sigma}\end{bmatrix}\Big)\in\Omega\subseteq \mathbb{R}^{n(q+pn+1)\times 1} $ represents the complete set of unknown parameters described by models \eqref{A1} and \eqref{A3}.
\subsection{Partial correlation graphs}\label{sec:partial_graph}
Let $G = (\mathscr{V}, \mathscr{E}, \overrightarrow{\mathscr{E}})$ be a mixed graph (graph with directed and undirected edges), where $\mathscr{V}$ is a set of vertices, $ \mathscr{E} $ is a set of undirected edges $ \mathscr{E}\subseteq \mathscr{V}\times \mathscr{V} $ and $\overrightarrow{\mathscr{E}}$ is a set of directed edges $ \overrightarrow{\mathscr{E}} \subseteq \mathscr{V}\times \mathscr{V}$. The undirected edge between vertices $i$ and $j$ is denoted as $(i,j)$. If $(i,j)\in\mathscr{E}$ then also $(j,i)\in\mathscr{E}$. The directed edge between vertices $i$ and $j$ with direction from $i$ to $j$ is denoted as $(i,j)$. If $(i,j)\in\overrightarrow{\mathscr{E}}$ then $(j,i)\notin\overrightarrow{\mathscr{E}}$. 

Let $ \mathbf{X}_{\mathscr{V}\backslash\{i,j\}}\in\mathbb{R}^{(n-2)} $ denote $ \mathbf{X}(t) $ except the $i$th and $j$th elements i.e., $ X_{i}(t) $ and $ X_{j}(t) $ and $ \mathbf{Y}_{\mathscr{V}\backslash\{i,j\}}\in\mathbb{R}^{(n-2)} $ denote $ \mathbf{Y}(t) $ except $ Y_{i}(t) $ and $ Y_{j}(t) $. The partial correlation between any two time series of multivariate second order stationary processes is given in \citep{Graphical-Models3}. Define
\begin{equation}\label{A14}
\mathcal{E}_{i}(t)=X_{i}(t)-\mu_{i}^{*}-\sum_{u}\mathbf{a}^{*}_{i}(t-u)\mathbf{X}_{\mathscr{V}\backslash\{i,j\}},
\end{equation}
as the error process between $X_{i}(t)$ and best linear filter based on $ \mathbf{X}_{\mathscr{V}\backslash\{i,j\}} $ that minimizes $E\big(\mathcal{E}_{i}(t)\big)^{2}$. Here $\mu_{i}^{*}$ and $\{\mathbf{a}^{*}_{i}(u)\}$ are the optimal linear filters, exact expressions for which are available in \citep{Time-Series}. Similarly, the error process for $X_{j}(t)$ is $\mathcal{E}_{j}(t)$.
If the cross-covariance between $ \mathcal{E}_{i}(t) $ and $ \mathcal{E}_{j}(t) $ for all time lag $h\in \mathbb{Z}$ is zero, i.e., $Cov(\mathcal{E}_{i}(t), \mathcal{E}_{j}(t-h)) = 0$, $\forall h\in \mathbb{Z}$, then $ X_{i}(t) $ and $ X_{j}(t) $ are defined to be partially uncorrelated given $\mathbf{X}_{\mathscr{V}\backslash\{i,j\}} $. 

From \citep{Granger}, the process $X_i(t)$ is said to be causal for the another process $X_j(t)$ if the prediction of $X_j(t)$ using all available information except $X_i(t)$, can be improved by adding the available information about $X_i(t)$. 

We define the partial correlations and causalities between observed time series of counts in terms of the partial correlations and causalities between second order stationary latent processes as follows.
\begin{definition}\label{Def_new}
	Two time series $ Y_{i}(t) $ and $ Y_{j}(t) $ are partially uncorrelated given $ \mathbf{Y}_{\mathscr{V}\backslash\{i,j\}} $ if $ X_{i}(t) $ and $ X_{j}(t) $ are partially uncorrelated given $ \mathbf{X}_{\mathscr{V}\backslash\{i,j\}} $ and $Y_{i}(t)$ is not causal for $Y_{j}(t)$ if $X_{i}(t)$ is not causal for $X_{j}(t)$.
\end{definition}
Now, the partial correlation and causality graph for the multivariate time series $\mathbf{Y}(t)$ is defined as,
\begin{definition}[Partial correlation and causality graph of $\mathbf{Y}(t)$]
	The Partial correlation and causality graph of $\mathbf{Y}(t)$ modeled in \eqref{A1} and \eqref{A3} is the mixed graph $G = (\mathscr{V}, \mathscr{E},  \overrightarrow{\mathscr{E}})$ with vertex set $ \mathscr{V}  = \{Y_{1}(t),\dots, Y_{n}(t)\} $, undirected edge set $ \mathscr{E} $ and directed edge set  $\overrightarrow{\mathscr{E}}$ such that there is no undirected edge between the $i^{th}$ and the $j^{th}$ node if and only if the corresponding latent variables $ X_{i}(t) $ and $ X_{j}(t) $ are partially uncorrelated given $  \mathbf{X}_{\mathscr{V}\backslash\{i,j\}} $, and there is no directed edge from the $i^{th}$ to the $j^{th}$ node if and only if $X_{i}(t)$ is not causal for $X_{j}(t)$.
\end{definition}
Let $R_{XX}(h) =E[(\mathbf{X}(t))(\mathbf{X}(t-h))^{T}]$ be the autocovariance function of $\mathbf{X}(t)$ with time lag $h\in\mathbb{Z}$ and $S_{XX}(h)$ be the spectral density matrix of $ \mathbf{X}(t) $ i.e.,
\begin{align}\label{SDM}  
S_{XX}(\omega) = \sum_{h=-\infty}^{\infty}R_{XX}(h)e^{-jh\omega}.
\end{align}
It is known that the partial correlation graph can be obtained from the inverse spectral density matrix of multivariate second order stationary process \citep{Time-Series} which is given in the theorem below.
\begin{theorem}\label{par_bas}\text{\citep{Time-Series}}
	Consider the stationary time series $ \mathbf{X}(t)\in \mathbb{R}^{n} $ and assume that the spectral density matrix $S_{XX}(\omega)$ of $ \mathbf{X}(t) $ is invertible for all $\omega$. Then, $ X_{i}(t) $ and $ X_{j}(t) $ are partially uncorrelated given $ \mathbf{X}_{\mathscr{V}\backslash\{i,j\}} $ if and only if $((S_{XX}(\omega))^{-1})_{ij} = 0, \forall \omega$.
\end{theorem}
Thus from \citep{Songsiri:2010} and Theorem \ref{par_bas}, the partial correlation relations for the stationary latent process $ \mathbf{X}(t) $, which follows an AR model \eqref{A3}, in terms of inverse spectral are given in the lemma below.
\begin{lemma}\label{lamma_par} 
	\begin{equation*}\label{conditional indp constriants}  
	(S_{XX}(\omega)^{-1})_{ij} = 0\Leftrightarrow (W_{k})_{ij} = 0~\text{and}~(W_{k})_{ji} = 0~\text{for}~k = 0,\dots,p,    
	\end{equation*}
	where
	\begin{equation}\label{partial_corr}
	W_{k} = \begin{cases}
	-\Sigma^{-1}+\sum_{l=1}^{p} A_{l}^{T}\Sigma^{-1}A_{l} , & k = 0\\
	-2\Sigma^{-1}A_{k}+2\sum_{l=1}^{p-k} A_{l}^{T}\Sigma^{-1}A_{l+k}, & k = 1,\dots,p
	\end{cases}
	\end{equation}
	with $A_{0} = I$. 
\end{lemma} 
The following result is evident from Definition \ref{Def_new}, Theorem \ref{par_bas} and \eqref{conditional indp constriants}.
\begin{proposition}\label{par_corr}
	$ Y_{i}(t) $ and $ Y_{j}(t) $ are partially uncorrelated given $ \mathbf{Y}_{\mathscr{V}\backslash\{i,j\}} $ if $(\boldsymbol{W}_{k})_{ij} = 0~~\text{and}~~ (\boldsymbol{W}_{k})_{ji} = 0, ~\text{for}~ k = 0,\dots,p$.
\end{proposition} 
The following theorem states that the causality graph of the second order stationary process $X(t)$ can be found in terms of AR coefficients given in \eqref{A3} (\citet{Tjstheim,HSIAO1982}). 
\begin{theorem}\label{causal_AR}
	The time series $X_{i}(t)$ does not cause $X_{j}(t)$ if and only if the corresponding components $A_{k}(j,i)$ vanish for all $k$ i.e.,
	\begin{align} 
	X_{i}(t)~\text{does not cause}~ X_{j}(t)\Longleftrightarrow A_{k}(j,i)=0, ~\forall k\in\{1,\dots,p\}.
	\end{align}
\end{theorem}
The following definitions are useful for future development in later sections. Let $IW_i$ be the total incoming edge weight of $i^{th}$ node from other nodes in the causality graph, i.e.,
\begin{align}\label{inc_wei}
IW_{i} = \sum_{k=1}^{p}\sum_{j=1,j\neq i}^{n}A_{k}(i,j)
\end{align}
Similarly, let $OW_i$ be the total outgoing edge weight of $i^{th}$ node to other nodes in the causality graph, i.e.,
\begin{align}\label{out_wei}
OW_{i} = \sum_{k=1}^{p}\sum_{j=1,j\neq i}^{n}A_{k}(j,i)
\end{align}

\subsection{Problem formulation}\label{sec:problem}
From Theorem \ref{par_bas} and Proposition \ref{par_corr}, we need to estimate the inverse spectral density matrix of multivariate latent process $ \mathbf{X}(t) $ to estimate the partial correlations between the observed multiple time series of counts $ \mathbf{Y}(t) $. Since the number of possible edges in the partial correlation and causality graph can be large (up to $n^{2}$), overfitting can be a potential issue in this estimation problem \citep{overfitting}.
LASSO \citep{LASSO} is one of the regularization methods to overcome this difficulty by setting certain parameters to zero. This technique is used here to introduce sparsity in the inverse spectral density matrix. In LASSO, the maximum likelihood estimation problem is regularized using the $\ell_{1}$-norm of the partial correlation constraints given in Proposition \ref{par_corr}. The $\ell_{1}$-norm of a vector $\mathbf{x}=\begin{bmatrix}x_{1}&\dots&x_{n}\end{bmatrix}^{T}\in \mathbb{R}^{n}$ is defined as $\Vert\mathbf{x}\Vert_{1}=\vert x_{1}\vert+\dots+\vert x_{n}\vert$. Let $\mathbf{Y}:=\begin{bmatrix}\mathbf{Y}(1)\dots \mathbf{Y}(N)\end{bmatrix}\in\mathbb{R}^{n\times N}$ be the observed data matrix and $\mathbf{X}:=\begin{bmatrix}\mathbf{X}(1)\dots \mathbf{X}(N)\end{bmatrix}\in\mathbb{R}^{n\times N}$ be the latent process data matrix. 
From \citep{Bayesian:1}, the joint probability density function (PDF) of $\mathbf{Y}$ and $\mathbf{X}$ is
\begin{align}\label{conditional_PDF}
&P(\mathbf{Y},\mathbf{X},\bs{\theta})= P(\mathbf{Y},\bs{\theta}|\mathbf{X})P(\mathbf{X},\bs{\theta})\nonumber\\
&=\Big\{\prod_{i=1}^{n}\prod_{t=1}^{N}P\big(Y_{i}(t),\bs{\theta}|X_{i}(t)\big)\prod_{t=p+1}^{N}P\big(\mathbf{X}(t), \bs{\theta}|\mathbf{X}(t-1),\nonumber\\
&\qquad\qquad\qquad\qquad\qquad\dots,\mathbf{X}(t-p)\big)\Big\} P\big(\mathbf{X}(1),\dots,\mathbf{X}(p),\bs{\theta}\big)\nonumber\\&= \prod_{i=1}^{n}\prod_{t=1}^{N}\text{exp}\big(X_{i}(t)Y_{i}(t)+\mathbf{z}_{t,i}^{T}\bs{\beta}_{i} Y_{i}(t)-e^{X_{i}(t)+\mathbf{z}_{t,i}^{T}\bs{\beta}_{i}}\big)\nonumber\\&\qquad\qquad\qquad\qquad\prod_{t=p+1}^{N}\frac{1}{\sqrt{2\pi|\Sigma|}}\text{exp}\Big\{-\frac{1}{2}\big(\epsilon(t)^{T}\Sigma^{-1}\epsilon(t)\big)\Big\}\nonumber\\&\qquad\qquad\qquad\qquad P\big(\mathbf{X}(1),\dots,\mathbf{X}(p),\bs{\theta}\big).
\end{align}
where $\epsilon(t)=\mathbf{X}(t) - \sum_{k = 1}^{p}A_{k}\mathbf{X}(t-k)$ and $P\big(\mathbf{X}(1),\dots,\mathbf{X}(p),\bs{\theta}\big)$ is the joint probability density function of initial values $\mathbf{X}(1),\dots,\mathbf{X}(p)$. The calculation of probability density function $P\big(\mathbf{X}(1),\dots,\mathbf{X}(p),\bs{\theta}\big)$ is given in \citep{helmut2005new}, p.29. Let
\begin{align}
\mathscr{X}(t)&:=\begin{bmatrix}\mathbf{X}(t)\\\mathbf{X}(t-1)\\\vdots\\\mathbf{X}(t-p+1)\end{bmatrix}\in\mathbb{R}^{np\times 1},\quad \mathscr{W}(t) := \begin{bmatrix}\epsilon(t)\\0\\\vdots\\0\end{bmatrix}\in\mathbb{R}^{np\times 1},\nonumber\\
\mathbf{A}&:=\begin{bmatrix}A_{1}&A_{2}&\dots&A_{p-1}&A_{p}\\I_{n}&0&\dots&0&0\\0&I_{n}&\dots&0&0\\\vdots&\vdots&\cdots&\vdots&\vdots\\0&0&\dots&I_{n}&0\end{bmatrix}\in\mathbb{R}^{np\times np}.
\end{align}
Then, the model \eqref{A3} becomes
\begin{align}\label{big_model}
\mathscr{X}(t) = \mathbf{A}\mathscr{X}(t-1)+\mathscr{W}(t).
\end{align}
Note that $\mathbf{X}(t)$ is a second-order stationary process, and $\mathbf{X}(t)$ and $\epsilon(t)$ are uncorrelated. Thus, from \eqref{big_model} the auto-covariance function of $\mathscr{X}(t)$ with zero time lag is
\begin{align*}
E[\mathscr{X}(t)\mathscr{X}(t)^{T}] = \mathbf{A}E[\mathscr{X}(t-1)\mathscr{X}(t-1)^{T}]\mathbf{A}^{T} + E[\mathscr{W}(t)\mathscr{W}(t)^{T}],
\end{align*}
\begin{align}\label{auto}
R_{\mathscr{X}\mathscr{X}}(0) = \mathbf{A}R_{\mathscr{X}\mathscr{X}}(0)\mathbf{A}^{T} + R_{\mathscr{W}\mathscr{W}}(0),
\end{align}
where 
\begin{align}
R_{\mathscr{X}\mathscr{X}}(0)&=\begin{bmatrix}R_{XX}(0)~~~R_{XX}(1)~\dots~ R_{XX}(p-1)\\R_{XX}(1)^{T}~~R_{XX}(0)~\dots~ R_{XX}(p-2)\\\vdots~~~~~~~~~~\vdots~~~~~~\cdots~~~~~~\vdots\\R_{XX}(p-1)^{T}~R_{XX}(p-2)^{T}~\dots~ R_{XX}(0)\end{bmatrix}\in\mathbb{R}^{np\times np},\nonumber\\
R_{\mathscr{W}\mathscr{W}}(0)&=\begin{bmatrix}\Sigma&0&\dots&0\\0&0&\dots& 0\\\vdots&\vdots&\cdots&\vdots\\0&0&\dots&0\end{bmatrix}\in\mathbb{R}^{np\times np}.
\end{align}
The vectorized equation of \eqref{auto} is
\begin{align}\label{auto1}
\text{vec}(R_{\mathscr{X}\mathscr{X}}(0)) &= (\mathbf{A}\otimes\mathbf{A})\text{vec}(R_{\mathscr{X}\mathscr{X}}(0)) +\text{vec}(R_{\mathscr{W}\mathscr{W}}(0)),\nonumber\\
\text{vec}(R_{\mathscr{X}\mathscr{X}}(0)) &= (I_{np} - \mathbf{A}\otimes\mathbf{A})^{-1}\text{vec}(R_{\mathscr{W}\mathscr{W}}(0)),
\end{align} 
where `$\otimes$' denotes the Kronecker product. Therefore, the initial values $\mathbf{X}(1),\dots,\mathbf{X}(p)$ follows a Gaussian distribution \citep{helmut2005new} with zero mean and variance is $R_{\mathscr{X}\mathscr{X}}(0)$ which can be calculated from \eqref{auto1}. 

The exact log-likelihood function of $\mathbf{Y}$ and $\mathbf{X}$ is $l(\mathbf{Y},\mathbf{X},\bs{\theta}) := log(P(\mathbf{Y},\mathbf{X},\bs{\theta}))$.
Let 
\begin{align}\label{complete_lik}
L(\mathbf{Y}, \mathbf{X}, \bs{\theta}):=l(\mathbf{Y}, \mathbf{X}, \bs{\theta})-\gamma h_{1}(W_{0},W_{1},\hdots, W_{p})
\end{align}
where $\gamma\geq 0$ is a regularization parameter and the regularization function $h_{1}(W_{0},W_{1},\hdots, W_{p}) = \sum_{j>i}^{} \sum_{k = 0}^{p}\big\{|(W_{k})_{ij}|+|(W_{k})_{ji}|\big\}$ is 
the $\ell_{1}$-norm of off-diagonal elements of matrices $W_{0},W_{1},\hdots, W_{p}$ given in \eqref{partial_corr}.
The amount of sparsity of the estimated inverse spectral density matrix is controlled by the regularization parameter $\gamma$. As $\gamma$ varies, the sparsity pattern varies in the estimated inverse spectral density matrix from dense ($\gamma $ small) to diagonal ($\gamma $ large). We aim to maximize regularized log-likelihood $L(\mathbf{Y}, \mathbf{X}, \bs{\theta})$ with respect to $\bs{\theta}$ over $\Omega$ to estimate the optimal parameters. To maximize $L(\mathbf{Y}, \mathbf{X}, \bs{\theta})$ directly, data matrix $\mathbf{X}$ should be known. However, in reality the latent data matrix $\mathbf{X}$ is unknown. The marginal PDF of $\mathbf{Y}$ is $P(\mathbf{Y}, \bs{\theta}):=\int_{\mathbf{X}}P(\mathbf{Y},\mathbf{X}, \bs{\theta})d\mathbf{X}$. Let $l(\mathbf{Y}, \bs{\theta}):=log(P(\mathbf{Y}, \bs{\theta}))$ be the marginal log-likelihood of $\mathbf{Y}$. The marginal log-likelihood of $\mathbf{Y}$ with $\ell_{1}$-type regularization 
\begin{align}\label{marginal_lik}
L(\mathbf{Y}, \bs{\theta}):=l(\mathbf{Y}, \bs{\theta}) -\gamma h_{1}(W_{0},W_{1},\hdots, W_{p}),
\end{align} 
is used to estimate $\bs{\theta}$. We next formulate the optimization problem to estimate $\bs{\theta}$.
\begin{problem}\label{P1}
	Find
	\begin{align}
	\bs{\theta}^{*} = \text{arg}\hspace{0.3mm}\max_{\bs{\theta}\in\Omega}~L(\mathbf{Y}, \bs{\theta}).
	\end{align}	
\end{problem}

\subsection{Challenges in solution of Problem \ref{P1}}\label{Challenges}
To solve Problem \ref{P1} directly, we need to calculate $L(\mathbf{Y}, \bs{\theta})$ explicitly, which can be difficult because of $nN$ multiple integrals. To overcome this difficulty and unknown latent data matrix $\mathbf{X}$, we use the expectation and maximization (EM) algorithm \citep{emalgo} to solve Problem \ref{P1}. 
The EM algorithm maximizes the function $L(\mathbf{Y}, \bs{\theta})$ by working with $L(\mathbf{Y}, \mathbf{X}, \bs{\theta})$ and the conditional density function $P(\mathbf{X},\bs{\theta}|\mathbf{Y})$ of the latent data matrix $\mathbf{X}$ given observed count data matrix $\mathbf{Y}$.

Note that the standard EM algorithm is modified with the regularization term in the following Algorithm \ref{EM_algo}. The tolerance ($\delta$) is adjusted according to the accuracy requirement in the estimated parameters.
\begin{algorithm}[h!]
	\caption{Expectation and Maximization (EM) algorithm with $\ell_{1}$-type regularization}
	\label{EM_algo}
	\begin{algorithmic}
		\State \textit{\bf Input:} For $k=1$, initial condition $\bs{\theta}^{(1)}$. 
		\State \textit{\bf Output: } The estimated parameter is $\bs{\theta}^{*}$. 
	\end{algorithmic}
	\begin{algorithmic}[1]
		\While{$tol>\delta$}
		\State {\it{E-step}}:
		\begin{align*}
		\hspace{-0.5cm}Q(&\bs{\theta};\bs{\theta}^{(k)})=E_{\bs{\theta}^{(k)}}\{l(\mathbf{Y}, \mathbf{X},\bs{\theta})|\mathbf{Y}\}-\gamma h_{1}(
		W_{0},W_{1},\hdots, W_{p}),\nonumber\\&=\int_{X}^{}l(\mathbf{Y}, \mathbf{X},\bs{\theta})P(\mathbf{X},\bs{\theta}^{(k)}|\mathbf{Y})d\mathbf{X}-\gamma h_{1}(
		W_{0},W_{1},\hdots, W_{p}).
		\end{align*}
		\State {\it{M-step}}: Maximize the function $Q(\bs{\theta};\bs{\theta}^{(k)})$ with respect to $ \bs{\theta} $ over space $\Omega$ and update the maximizer,
		\begin{align}
		\bs{\theta}^{(k+1)}= \text{arg}\hspace{0.3mm}\max_{\bs{\theta}\in\Omega}^{}Q(\bs{\theta};\bs{\theta}^{(k)}).
		\end{align}
		\State $tol=\frac{\Vert\bs{\theta}^{(k+1)}-\bs{\theta}^{(k)}\Vert}{\Vert\bs{\theta}^{(k)}\Vert}$.
		\State $k=k+1$.
		\EndWhile
	\end{algorithmic}
	\begin{algorithmic}
		\State $\bs{\theta}^{*}=\bs{\theta}^{(k)}$. 
	\end{algorithmic}
\end{algorithm}

The log-likelihood function $l(\mathbf{Y}, \mathbf{X}, \bs{\theta})$ is required to be known to solve the Algorithm \ref{EM_algo}. However, the conditional density function $P(\mathbf{X},\bs{\theta}^{(k)}|\mathbf{Y})$ of the latent data matrix $\mathbf{X}$ given observed count data matrix $\mathbf{Y}$ is impossible to calculate at the $k$th step in the Algorithm \ref{EM_algo} because it is a mixture of Gaussian and Poisson distributions. Here Algorithm \ref{EM_algo} cannot be implemented directly. To overcome this difficulty, Monte Carlo techniques are used to draw samples from $P(\mathbf{X},\bs{\theta}^{(k)}|\mathbf{Y})$ and thereby numerically compute the integral given in the E-step of Algorithm \ref{EM_algo}. The  Metropolis-Hastings algorithm \citep{casella2002statistical} is used to generate the samples from $P(\mathbf{X},\bs{\theta}^{(k)}|\mathbf{Y})$.

The generation of samples from $P(\mathbf{X},\bs{\theta}^{(k)}|\mathbf{Y})$ using Metropolis-Hastings algorithm for the parameter driven model with AR(1) process is given in \citep{Bayesian:1}. In \citep{Bayesian:1}, the proposal distributions are calculated in Metropolis-Hastings algorithm with AR(1) process. Similarly, in this paper the proposal distributions for Metropolis-Hastings algorithm with AR(p) process are given. Let
\begin{align}
\Lambda&:=R_{\mathscr{X}\mathscr{X}}(0)^{-1}\nonumber\\
&:=\begin{bmatrix}\Lambda_{11}&\Lambda_{12}&\dots&\Lambda_{1p}\\\Lambda_{12}^{T}&\Lambda_{22}&\dots& \Lambda_{2p}\\\vdots&\vdots&\cdots&\vdots\\\Lambda_{1p}^{T}&\Lambda_{1(p-1)}^{T}&\dots&\Lambda_{pp}\end{bmatrix}.
\end{align}
For $t =1,\dots,p$, the proposal distribution of $\mathbf{X}(t)$ given $\mathbf{X}(1),\dots,\mathbf{X}(t-1)$ from the past, and $\mathbf{X}(t+1),\dots,\mathbf{X}(t+p)$ to the future follows a Gaussian distribution with variance,
\begin{align}\label{proposal1}
\Sigma_{t} = \Big(\Lambda_{tt}^{-1}+\sum_{i=0}^{t-1}A_{p-i}^{T}\Sigma^{-1}A_{p-i}\Big)^{-1},
\end{align}
and mean,
\begin{align}\label{proposal2}
\mu_{t}&= \Sigma_{t}\Big(\sum_{i=1}^{t}A_{p+1-i}^{T}\Sigma^{-1}\mathbf{X}(p+t+1-i)-\sum_{\substack{i=1\\i\neq t}}^{p}\Lambda_{ti}\mathbf{X}(i)\nonumber\\
&\quad-\sum_{j=1}^{t}\sum_{\substack{i=1\\i\neq p+1-j}}^{p}A_{p+1-j}^{T}\Sigma^{-1}A_{i}\mathbf{X}(p+t+1-j-i)\Big).
\end{align}   
For $t =p+1,\dots,N-p$, the proposal distribution of $\mathbf{X}(t)$ given $\mathbf{X}(t-1),\dots,\mathbf{X}(t-p)$ from the past, and $\mathbf{X}(t+1),\dots,\mathbf{X}(t+p)$ to the future follows a Gaussian distribution with variance,
\begin{align}\label{proposal3}
\Sigma_{t} = \Big(\Sigma^{-1}+\sum_{i=1}^{p}A_{i}^{T}\Sigma^{-1}A_{i}\Big)^{-1},
\end{align}
and mean,
\begin{align}\label{proposal4}
\mu_{t}&= \Sigma_{t}\Big(\sum_{i=1}^{p}\Sigma^{-1}A_{i}\mathbf{X}(t-i)+\sum_{i=1}^{p}A_{i}^{T}\Sigma^{-1}\mathbf{X}(t+i)\nonumber\\
&\quad-\sum_{j=1}^{p}\sum_{\substack{i=1\\i\neq j}}^{p}A_{j}^{T}\Sigma^{-1}A_{i}\mathbf{X}(t+j-i)\Big).
\end{align}
For $t=N-p+1,\dots,N$, the proposal distribution of $\mathbf{X}(t)$ given $\mathbf{X}(t-1),\dots,\mathbf{X}(t-p)$ from the past, and $\mathbf{X}(t+1),\dots,\mathbf{X}(N)$ to the future follows a Gaussian distribution with variance,
\begin{align}\label{proposal5}
\Sigma_{t} = \Big(\Sigma^{-1}+\sum_{i=1}^{N-t}A_{i}^{T}\Sigma^{-1}A_{i}\Big)^{-1},
\end{align}
and mean,
\begin{align}\label{proposal6}
\mu_{t}&= \Sigma_{t}\Big(\sum_{i=1}^{p}\Sigma^{-1}A_{i}\mathbf{X}(t-i)+\sum_{i=1}^{N-t}A_{i}^{T}\Sigma^{-1}\mathbf{X}(t+i)\nonumber\\
&\quad-\sum_{j=1}^{N-t}\sum_{\substack{i=1\\i\neq j}}^{p}A_{j}^{T}\Sigma^{-1}A_{i}\mathbf{X}(t+j-i)\Big).
\end{align}
Using these proposal distributions in Metropolis-Hastings algorithm given in Algorithm \ref{MH_algorithm}, we generate $m$ samples $\mathbf{X}^{(1)}$, $\mathbf{X}^{(2)}$,\dots, $\mathbf{X}^{(m)}\in\mathbb{R}^{n\times N}$ from conditional density function $P(\mathbf{X},\bs{\theta}^{(k)}|\mathbf{Y})$. Thus, we propose a Monte Carlo expectation and maximization (MCEM) algorithm with $\ell_{1}$-type regularization next. 
\begin{algorithm}[h!]
	\caption{Metropolis-Hastings algorithm}
	\label{MH_algorithm}
	\begin{algorithmic}
		\State \textit{\bf Input:} Initial values $\mathbf{X}^{(1)}:=[\mathbf{X}^{(1)}(1)~\mathbf{X}^{(1)}(2)\cdots\mathbf{X}^{(1)}(N)]\in\mathbb{R}^{n\times N}$ and number of samples $m$.
		\State \textit{\bf Output: } Generated samples $\mathbf{X}^{(1)}, \mathbf{X}^{(2)},\dots, \mathbf{X}^{(m)}\in\mathbb{R}^{n\times N}$.
	\end{algorithmic}
	\begin{algorithmic}[1]
		\For{$r=1,\dots,(m-1)$}
		\For{$t=1,\dots,N$}
		\For{$i=1,\dots,n$} 
		\State Generate uniform random number $ U_{i} $ and generate $ X_{i}(t)$ from the proposal distributions \eqref{proposal1}-\eqref{proposal6}.
		\State Calculate the acceptance probability,
		\begin{align*}
		\rho_{i}=\text{min}\bigg\{\frac{\text{exp}(X_{i}(t))Y_{i}(t)-\text{exp}(X_{i}(t)+\mathbf{z}_{t,i}^{T}\bs{\beta}_{i})}{\text{exp}(X_{i}^{(r)}(t))Y_{i}(t)-\text{exp}(X_{i}^{(r)}(t)+\mathbf{z}_{t,i}^{T}\bs{\beta}_{i})},1\bigg\},
		\end{align*}
		where $X_{i}^{(r)}(t)$ is the $i$th element of $ \mathbf{X}^{(r)}(t) $.
		\State Update $X^{(r+1)}_{i}(t)$:
		\begin{align*}
		X^{(r+1)}_{i}(t)= 
		\begin{cases}
		X_{i}(t),& \text{if}~U_{i}\leq\rho_{i}\\
		X_{i}^{(r)}(t),& \text{otherwise}
		\end{cases}.
		\end{align*}
		\EndFor
		\EndFor
		\EndFor
	\end{algorithmic}
\end{algorithm}


\section{Main results}\label{Main_results}
In this section we discuss the Monte Carlo expectation and maximization (MCEM) algorithm with $\ell_{1}$-type regularization. Further, the asymptotic convergence of sequence of parameters generated by the MCEM algorithm with $\ell_{1}$-type regularization is proved. 
\subsection{Monte Carlo Expectation and Maximization (MCEM) algorithm with $\ell_{1}$-type regularization}\label{MCEM_algo}
In this section, we propose a MCEM algorithm with $\ell_{1}$-type regularization to estimate parameters $\bs{\theta}$. As mentioned above, we approximate the integral in the $E$-$step$ (step 2) of Algorithm \ref{EM_algo}. The MCEM algorithm with $\ell_{1}$-type regularization is presented in Algorithm \ref{MCEM_algorithm}. The tolerance ($\delta$) is adjusted according to the accuracy requirement in the estimated parameters.
\begin{algorithm}[]
	\caption{Monte Carlo Expectation and Maximization (MCEM) algorithm with $\ell_{1}$-type regularization}
	\label{MCEM_algorithm}
	\begin{algorithmic}
		\State \textit{\bf Input:} For $k=1$, initial condition $\bs{\theta}^{(1)}_{m}$. 
		\State \textit{\bf Output: } The estimated parameter is $\bs{\theta}^{*}$. 
	\end{algorithmic}
	\begin{algorithmic}[1]
		\While{$tol>\delta$}
		\State Generate $ m $ samples $ \mathbf{X}^{(1)},\mathbf{X}^{(2)},\dots, \mathbf{X}^{(m)}\in\mathbb{R}^{n\times N} $ from the conditional PDF $P(\mathbf{X},\bs{\theta}^{(k)}_{m}|\mathbf{Y})$ using Metropolis-hastings algorithm \ref{MH_algorithm}.
		\State {\it{E-step}}:
		\begin{equation*}\label{B49}
		\hspace{-0.4cm}Q_{m}(\bs{\theta};\bs{\theta}^{(k)}_{m})=\frac{1}{m}\sum_{i=1}^{m}l(\mathbf{Y}, \mathbf{X}^{(i)},\bs{\theta})-\gamma h_{1}(
		W_{0},W_{1},\hdots, W_{p}),
		\end{equation*}
		\State {\it{M-step}}: Maximize the function (\ref{B49}) with respect to $ \bs{\theta} $ over space $\Omega$ and update the maximizer,
		\begin{align}\label{B50}
		\bs{\theta}^{(k+1)}_{m} = \text{arg}\hspace{0.3mm}\max_{\bs{\theta}\in\Omega}^{}~Q_{m}(\bs{\theta};\bs{\theta}^{(k)}_{m}).
		\end{align}
		\State $tol=\frac{\Vert\bs{\theta}^{(k+1)}_{m}-\bs{\theta}^{(k)}_{m}\Vert}{\Vert\bs{\theta}^{(k)}_{m}\Vert}$.
		\State $k=k+1$.
		\EndWhile
	\end{algorithmic}
	\begin{algorithmic}
		\State $\bs{\theta}^{*}=\bs{\theta}^{(k)}_{m}$. 
	\end{algorithmic}
\end{algorithm}
From this algorithm, the generated sequence is \{$ \bs{\theta}^{(k)}_{m}: k = 1, 2,\dots $\} and the estimated parameter is $\bs{\theta}^{*}$. This estimated parameter $\bs{\theta}^{*}$ is the local maximizer of the function $L(\mathbf{Y}, \bs{\theta})$. We prove this claim in the next section. Note that the generated sequence \{$ \bs{\theta}^{(k)}_{m}: k = 1, 2,\dots $\} is a sequence of random variables.

\subsection{Asymptotic results}\label{Asym_res}

We obtain the sequence \{$ \bs{\theta}^{(k)}_{m}: k = 1, 2,\dots $\} from the Monte Carlo expectation and maximization (MCEM) algorithm with $\ell_{1}$-type regularization given in Algorithm \ref{MCEM_algorithm}. We need the following lemmas to prove the asymptotic convergence of this sequence. From step-3 of Algorithm \ref{EM_algo}, define $M(\bs{\theta}^{(k)}):=\text{arg}\hspace{0.3mm}\max_{\bs{\theta}\in\Omega}^{}Q(\bs{\theta};\bs{\theta}^{(k)})$ and assume that $M$ is a continuous function on $\Omega$.
\begin{lemma}\label{Lemma2}
	$\bs{\theta}^{(k+1)}_{m}$ converges in probability to $M(\bs{\theta}^{(k)}_{m})$ as $ m\rightarrow\infty $.
\end{lemma}
\begin{proof}
	The proof of this lemma follows from the convergence results of Metropolis-Hastings algorithm \citep{MH_conv1,MH_conv2,MH_conv3}.
\end{proof}
\begin{lemma}\label{Lemma3}
	Let \{$ \bs{\theta}^{(k)}: k = 1, 2,\dots$\} be a sequence generated by Algorithm \ref{EM_algo}. Then, the sequence $ \{L(\mathbf{Y},\bs{\theta}^{(k)})\} $ is a non-decreasing sequence.
\end{lemma}
\begin{proof}
	
	The proof follows from page 78 of \citep{EMalgorithm}.
	
\end{proof}

\begin{lemma}\label{Lemma4}
	Let \{$ \bs{\theta}^{(k)}_{m}: k = 1, 2,\dots$\} be a sequence generated by Algorithm \ref{MCEM_algorithm}. Then, 
	\begin{align}
	P(\{L(\mathbf{Y}, \bs{\theta}^{(k)}_{m})\leq L(\mathbf{Y}, \bs{\theta}^{(k+1)}_{m})\})\rightarrow 1~~as~~ m\rightarrow\infty.
	\end{align}
\end{lemma}
\begin{proof}
	Consider the function given in E-step (step-3) of Algorithm \ref{MCEM_algorithm},
	\begin{align}\label{Q_function}
	\hspace{-0.23cm}Q_{m}(\bs{\theta};\bs{\theta}^{(k)}_{m})=\frac{1}{m}\sum_{i=1}^{m}l(\mathbf{Y}, \mathbf{X}^{(i)},\bs{\theta})-\gamma h_{1}(W_{0},W_{1},\hdots, W_{p}).
	\end{align}
	The log-likelihood function $l(\mathbf{Y}, \mathbf{X}^{(i)},\bs{\theta})=l(\mathbf{Y}, \bs{\theta})+l(\mathbf{X}^{(i)}|\mathbf{Y},\bs{\theta})$ where $l(\mathbf{Y}, \bs{\theta}):=log(P(\mathbf{Y}, \bs{\theta}))$ and $l(\mathbf{X}^{(i)}|\mathbf{Y},\bs{\theta}):=log(P(\mathbf{X}^{(i)}|\mathbf{Y}, \bs{\theta}))$. Replace this joint log-likelihood in \eqref{Q_function},
	\begin{align}\label{Q_function1}
	Q_{m}(\bs{\theta};\bs{\theta}^{(k)}_{m})=\frac{1}{m}&\sum_{i=1}^{m}l(\mathbf{X}^{(i)}|\mathbf{Y},\bs{\theta})\nonumber\\&+ \big(l(\mathbf{Y}, \bs{\theta}) -\gamma h_{1}(W_{0},W_{1},\hdots, W_{p})\big).
	\end{align}
	Let $H_{m}(\bs{\theta};\bs{\theta}^{(k)}_{m}):=\frac{1}{m}\sum_{i=1}^{m}l(\mathbf{X}^{(i)}|\mathbf{Y},\bs{\theta})$. Then, \eqref{Q_function1} becomes
	\begin{align}\label{Q_function2}
	Q_{m}(\bs{\theta};\bs{\theta}^{(k)}_{m})=H_{m}(\bs{\theta};\bs{\theta}^{(k)}_{m})+L(\mathbf{Y}, \bs{\theta}).
	\end{align}
	The function $L(\mathbf{Y}, \bs{\theta})$ at $\bs{\theta}=\bs{\theta}^{(k)}_{m}$ in \eqref{Q_function2} is
	\begin{align}\label{Q_funk}
	L(\mathbf{Y}, \bs{\theta}^{(k)}_{m})=Q_{m}(\bs{\theta}^{(k)}_{m};\bs{\theta}^{(k)}_{m})-H_{m}(\bs{\theta}^{(k)}_{m};\bs{\theta}^{(k)}_{m}).
	\end{align}
	The function $L(\mathbf{Y}, \bs{\theta})$ at $\bs{\theta}=\bs{\theta}^{(k+1)}_{m}$ in \eqref{Q_function2} is
	\begin{align}\label{Q_funk1}
	L(\mathbf{Y}, \bs{\theta}^{(k+1)}_{m})=Q_{m}(\bs{\theta}^{(k+1)}_{m};\bs{\theta}^{(k)}_{m})-H_{m}(\bs{\theta}^{(k+1)}_{m};\bs{\theta}^{(k)}_{m}).
	\end{align}
	Subtract \eqref{Q_funk} from \eqref{Q_funk1}, then
	\begin{align}\label{Q_fundef}
	\{L&(\mathbf{Y}, \bs{\theta}^{(k+1)}_{m})-L(\mathbf{Y}, \bs{\theta}^{(k)}_{m})\}=\{Q_{m}(\bs{\theta}^{(k+1)}_{m};\bs{\theta}^{(k)}_{m})\nonumber\\&-Q_{m}(\bs{\theta}^{(k)}_{m};\bs{\theta}^{(k)}_{m})\}-\{H_{m}(\bs{\theta}^{(k+1)}_{m};\bs{\theta}^{(k)}_{m})-H_{m}(\bs{\theta}^{(k)}_{m};\bs{\theta}^{(k)}_{m})\}.
	\end{align}
	From Algorithm \ref{MCEM_algorithm}, we know that
	\begin{align}\label{ineq2}
	Q_{m}(\bs{\theta}^{(k+1)}_{m};\bs{\theta}^{(k)}_{m})\geq Q_{m}(\bs{\theta}^{(k)}_{m};\bs{\theta}^{(k)}_{m}).
	\end{align} 	
	Consider the second part of right hand side of \eqref{Q_fundef},
	\begin{align}\label{probeq}
	H_{m}(\bs{\theta}^{(k+1)}_{m};\bs{\theta}^{(k)}_{m})&-H_{m}(\bs{\theta}^{(k)}_{m};\bs{\theta}^{(k)}_{m})\nonumber\\&=\frac{1}{m}\sum_{i=1}^{m}\big(l(\mathbf{X}^{(i)}|\mathbf{Y},\bs{\theta}^{(k+1)}_{m})-l(\mathbf{X}^{(i)}|\mathbf{Y},\bs{\theta}^{(k)}_{m})\big)\nonumber\\&=\frac{1}{m}\sum_{i=1}^{m}log\Big\{\frac{P(\mathbf{X}^{(i)}|\mathbf{Y},\bs{\theta}^{(k+1)}_{m})}{P(\mathbf{X}^{(i)}|\mathbf{Y},\bs{\theta}^{(k)}_{m})}\Big\}\nonumber\\&\leq log\Big\{\frac{1}{m}\sum_{i=1}^{m}\frac{P(\mathbf{X}^{(i)}|\mathbf{Y},\bs{\theta}^{(k+1)}_{m})}{P(\mathbf{X}^{(i)}|\mathbf{Y},\bs{\theta}^{(k)}_{m})}\Big\}.
	\end{align}
	Consider right hand side of \eqref{Q_fundef}, from law of large numbers \citep{MH_conv3},
	\begin{align}\label{probeq1}
	\frac{1}{m}\sum_{i=1}^{m}\frac{P(\mathbf{X}^{(i)}|\mathbf{Y},\bs{\theta}^{(k+1)}_{m})}{P(\mathbf{X}^{(i)}|\mathbf{Y},\bs{\theta}^{(k)}_{m})}&\xrightarrow{p}\int_{X}\frac{P(\mathbf{X}|\mathbf{Y},\bs{\theta}^{(k+1)}_{m})}{P(\mathbf{X}|\mathbf{Y},\bs{\theta}^{(k)}_{m})}P(\mathbf{X}|\mathbf{Y},\bs{\theta}^{(k)}_{m})~d\mathbf{X}\nonumber\\&=\int_{X}P(\mathbf{X}|\mathbf{Y},\bs{\theta}^{(k+1)}_{m})~d\mathbf{X}\nonumber\\&=1.
	\end{align}
	as $m\rightarrow\infty$. Then the inequality \eqref{probeq} holds in probability,
	\begin{align}\label{probeq2}
	H_{m}(\bs{\theta}^{(k+1)}_{m};\bs{\theta}^{(k)}_{m})-H_{m}(\bs{\theta}^{(k)}_{m};\bs{\theta}^{(k)}_{m})\leq 0,
	\end{align}
	as $m\rightarrow\infty$. Therefore, from \eqref{ineq2} and \eqref{probeq2}, the equation \eqref{Q_fundef} becomes,
	\begin{align}\label{}
	L(\mathbf{Y}, \bs{\theta}^{(k+1)}_{m})\geq L(\mathbf{Y}, \bs{\theta}^{(k)}_{m}).
	\end{align}
	in probability as $m\rightarrow\infty$.
\end{proof}

Following \citep{10.2307/2291149}, Lemma \ref{Lemma2}, Lemma \ref{Lemma3} and Lemma \ref{Lemma4}, the asymptotic convergence results of sequence \{$ \bs{\theta}^{(k)}_{m}: k = 1, 2,\dots $\} follows from \citep{10.2307/2291149}.          
\begin{theorem}\label{theo2}
	Let \{$ \bs{\theta}^{(k)}_{m}: k = 1, 2,\dots$\} be a sequence generated by Algorithm \ref{MCEM_algorithm} based on sample size $ m $. Suppose $\bs{\theta}^{*}$ is an isolated local maximizer of function $L(\mathbf{Y}, \bs{\theta})$. 
	For any $ \epsilon>0 $, there exist $ K_{0}<\infty$ and $\delta>0$ such that for any starting value $\bs{\theta}^{(1)}_{m}\in\mathscr{N}:=\{\bs{\theta}: \Vert \bs{\theta}-\bs{\theta}^{*}\Vert\leq\delta\}$,
	\begin{align}\label{prob}
	P(\{\Vert \bs{\theta}^{(k)}_{m}-\bs{\theta}^{*}\Vert<\epsilon~\text{for~some}~k\leq K_{0}\})\rightarrow 1~~as~~ m\rightarrow\infty.
	\end{align}
\end{theorem}


\section{Inference on Simulated data}\label{sec:results}
The performance of the MCEM algorithm with $ \ell_{1} $-type regularization to estimate the partial correlation and causality graphs, is tested using randomly generated data in this section.

We consider a parameter driven model \eqref{A1} to generate counts randomly. The increasing trend and yearly seasonality terms are considered in the true model with number of time series $ n=10 $. The delay in the AR process \eqref{A3} is taken to be $ p=2 $. The conditional mean in the true parameter driven model \eqref{A1} is
\begin{align}\label{sim_mod}
\mu_{i}(t) = \text{exp}(\mathbf{z}_{t,i}^{T}\bs{\beta}_{i} + X_{i}(t))~\text{for}~i=1,2,\dots,10,
\end{align} 
where $ \mathbf{z}_{t,i}=\begin{bmatrix}1&t&\text{cos}(2\pi t/12)&\text{sin}(2\pi t/12)\end{bmatrix} $ and $\bs{\beta}_{i}  = \begin{bmatrix}0.5&0.005&0.5&0.5\end{bmatrix} $ for all $ i=1,2,\dots,10 $. The AR(2) multivariate latent process $\mathbf{X}(t)\in\R^{10}$ is
\begin{align}\label{sim_mod1}
\mathbf{X}(t) = \sum_{k = 1}^{2}A_{k}\mathbf{X}(t-k) + \mathbf{\epsilon}(t),\qquad \mathbf{\epsilon}(t)\sim N(\mathbf{0},\Sigma)
\end{align}
where $10\times10 $ matrices $ A_{1} $ and $ A_{2} $ are randomly chosen with elements $ \pm 0.3 $ and zeros. The covariance matrix of Gaussian noise, $ \Sigma $ is the diagonal matrix with diagonal elements equal to 0.01. We generate the Poisson counts for $N=200$ from this true model.  Let $\mathbf{Y}:=\begin{bmatrix}\mathbf{Y}(1)&\hdots&\mathbf{Y}(200)\end{bmatrix}$ be an observed data matrix. From the observed multivariate data $\mathbf{Y}$, the partial correlation and causality graph is estimated next. The partial correlations and causalities between multiple count time series for this true model \eqref{sim_mod} are shown in separate graphs given in Fig. \ref{true_par} and Fig. \ref{true_cas}. 

\subsection{Choice of regularization parameter $ \gamma $}\label{threshold}
The sparsity in the inverse spectral density matrix is controlled by the regularization parameter $ \gamma $. As $ \gamma $ varies, the sparsity in the inverse spectral density matrix changes. There are several methods to estimate the regularization parameter $ \gamma $. Cross-validation \citep{HLP06} is one of the methods to estimate $ \gamma $. This method is not accurate and it requires significant computations to get an accurate $ \gamma $ if the length of observed data is less. A method to select the `best' regularization parameter was given in \citep{Jitkomut} based on information scores. In this method, the regularization parameter is estimated using the trade-off curve \citep{Jitkomut} between the log-likelihood ($l(\mathbf{Y}, \mathbf{X}, \bs{\theta})$) and the $\ell_{1}$-type regularization function ($ h_{1}(W_{0},W_{1},\dots,W_{p}) $). We collect several $ \gamma $ values from this trade-off curve. We use a further thresholding on the elements of the inverse spectral density matrix for each value of $\gamma$ along the trade-off curve. This method is as follows.  

\noindent\textbf{Partial coherence spectrum :} The inverse spectral density matrix normalized with $ diag(S_{XX}(\omega)^{-1}) $ where $ diag(S_{XX}(\omega)^{-1}) $ is the diagonal of $ S_{XX}(\omega)^{-1} $, is called as the partial coherence spectrum $ R(\omega) $.
\begin{equation}\label{B53}
R(\omega) = diag(S_{XX}(\omega)^{-1})^{-1/2}S_{XX}(\omega)^{-1}diag(S_{XX}(\omega)^{-1})^{-1/2}
\end{equation} 
\textbf{Thresholding:} Let $\rho_{ij}$ be the $ \ell_{\infty} $-norm of the entries for the partial coherence spectrum $ R(\omega) $ with respect to $ \omega $ i.e., 
\begin{equation}\label{}       
\rho_{ij} = \sup_{\omega} |R(\omega)_{ij}|.
\end{equation}
The value of $ \rho_{ij} $ signifies the partial correlation between $ Y_{i}(t) $ and $ Y_{j}(t) $ given $ \mathbf{Y}_{\mathscr{V}\backslash\{i,j\}} $ and it ranges from 0 to 1. The value of $ \rho_{ij} $ being low signifies that the partial correlation between $ Y_{i}(t) $ and $ Y_{j}(t) $ given $ \mathbf{Y}_{\mathscr{V}\backslash\{i,j\}} $ is negligible. A thresholding approach to remove such negligible partial correlation has been discussed in \citep{lounici2008}. Let $\rho^{*}$ be a threshold value chosen by this method. If $ \rho_{ij} \leq \rho^{*}$, then remove the edge $(i, j)$ from the graph. Collect all the partial coherence spectrum $(R(\omega))$ corresponding $ \gamma $ values along the trade-off curve by applying the threshold. Then assign the ranks using Bayes information criteria (BIC) scores and select the partial correlation graph which has the lowest score.

\subsection{Topology selection from observed count data}
For the observed multivariate count data $\mathbf{Y}(t)$ for $t=1,\dots,200$, we applied Algorithm \ref{MCEM_algorithm}. We found $ \gamma $ values 0, 0.0698, 0.2911, 0.5857, 0.6872, 0.9963, 1.8527, 2.6891 and 3 from the trade-off curve which is given in Fig. \ref{Tradeoff_simulated}. For varies values of regularization parameter on the trade off curve, the partial coherence spectrum $R(\omega)$ from \eqref{B53} is calculated by setting entries with $\rho_{ij}\leq0.1$ to zero. Thus, all the partial coherence spectrum after thresholding corresponding to these $\gamma$ values are given in Fig.\ref{Graphs_simulated}. We observed that sparsity has increased as we move from $\gamma =0$ to $\gamma = 3$. Next we rank all these partial coherence spectrum with BIC score. The BIC scores corresponding to the $\gamma$ values are given in Fig. \ref{Lam_BIC_simulated}. From the Fig. \ref{Lam_BIC_simulated}, it is apparent that the best regularization parameter is selected to be $ \gamma^{*}  = 0.2911$. The estimated partial correlations and the causalities between multiple count time series corresponding to $\gamma^{*}  = 0.2911$ are shown in separate graphs given in Fig. \ref{est_par} and Fig. \ref{est_cas}. From the true and estimated partial correlation graphs in Fig. \ref{true_est_par}, it is observed that the proposed algorithm has misclassified only three edges as zeros and two edges as non-zeros. Similarly,  from the  true and estimated causality graphs Fig .\ref{true_est_cas}, it is seen that only four directed edges are incorrectly estimated.

\begin{figure}[h!]
	\centering
	\includegraphics[scale=0.2,center]{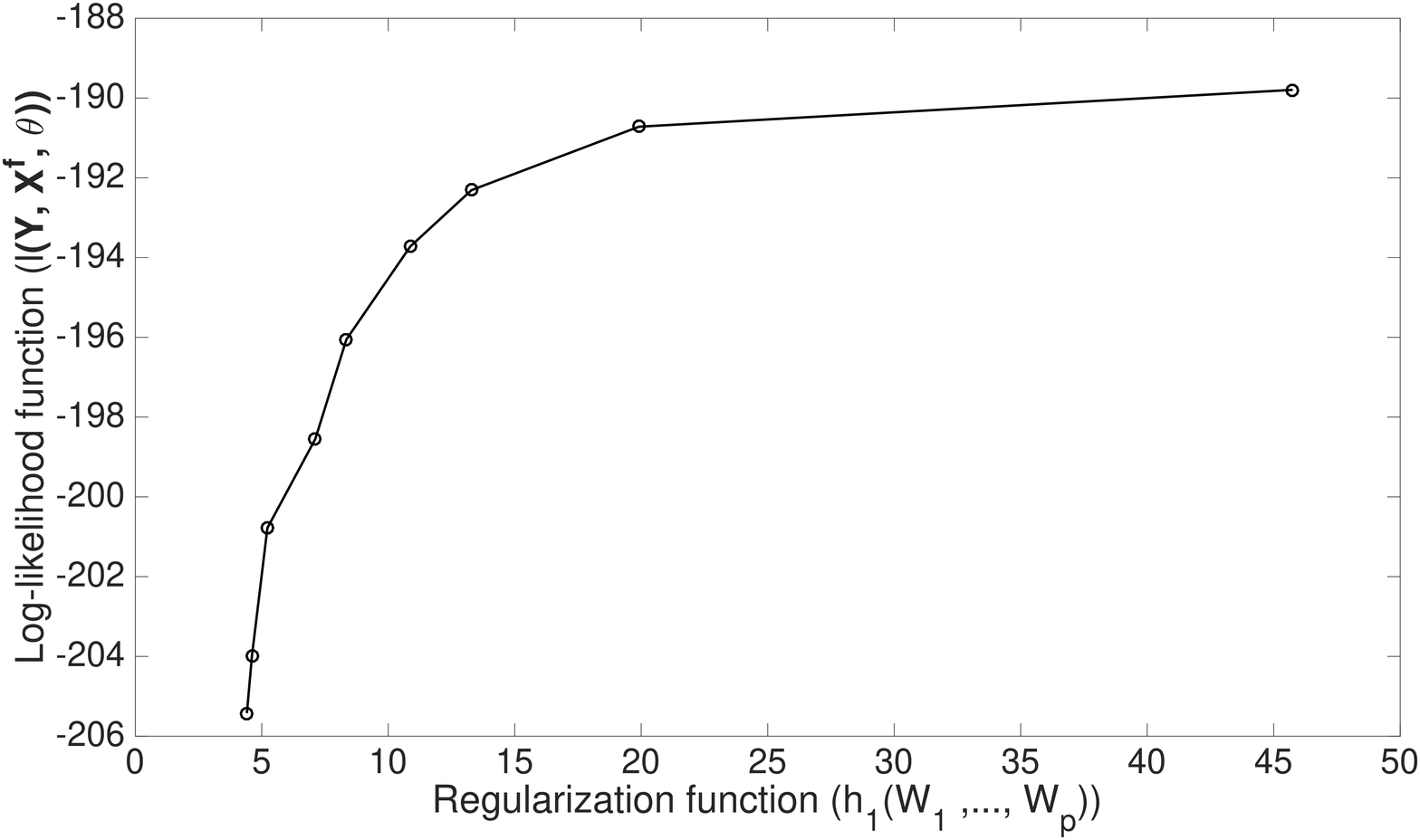}
	\caption{Trade-off curve between the conditional log-likelihood $l(\mathbf{Y}, \mathbf{X}, \bs{\theta})$ and  $ h_{1}(W_{0},W_{1},\dots,W_{p}) $}
	\label{Tradeoff_simulated}
\end{figure}
\begin{figure}[h!]
	\centering
	\includegraphics[scale=0.32,center]{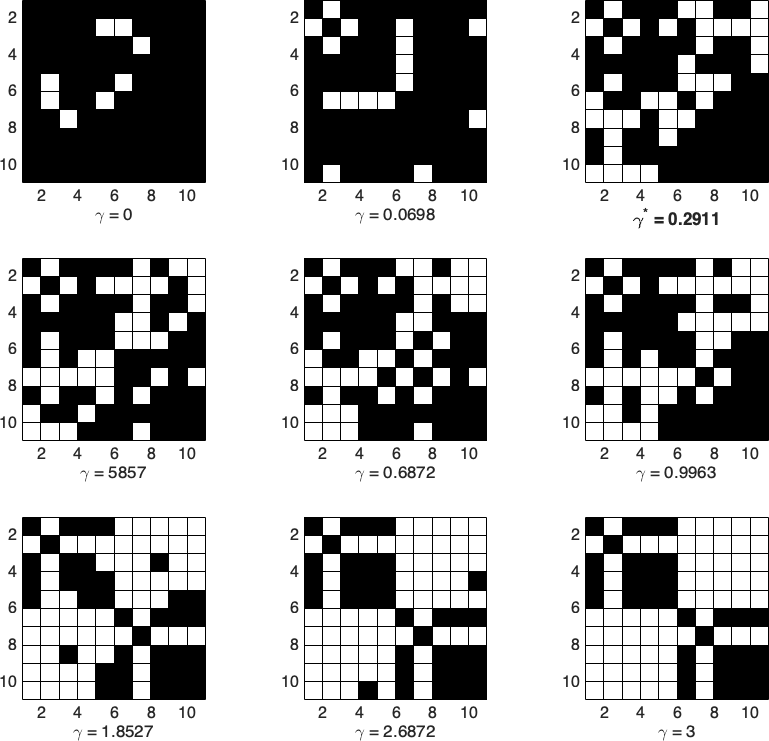}
	\caption{The partial correlation graphs along the trade of curve in Fig. \ref{Tradeoff_simulated}}
	\label{Graphs_simulated}
\end{figure}
\begin{figure}[h!]
	\centering
	\includegraphics[scale=0.2,center]{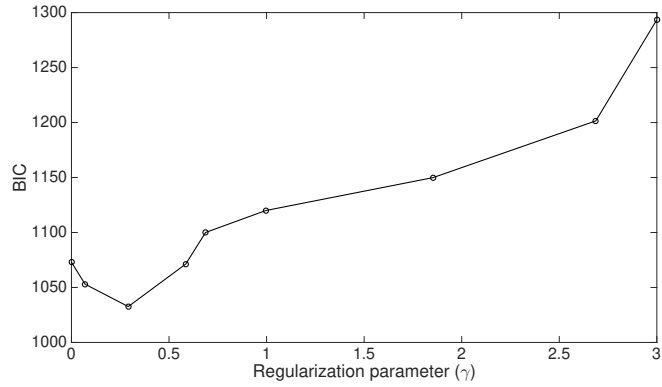}
	\caption{The BIC scores along the trade of curve in Fig. \ref{Tradeoff_simulated}}
	\label{Lam_BIC_simulated}
\end{figure}
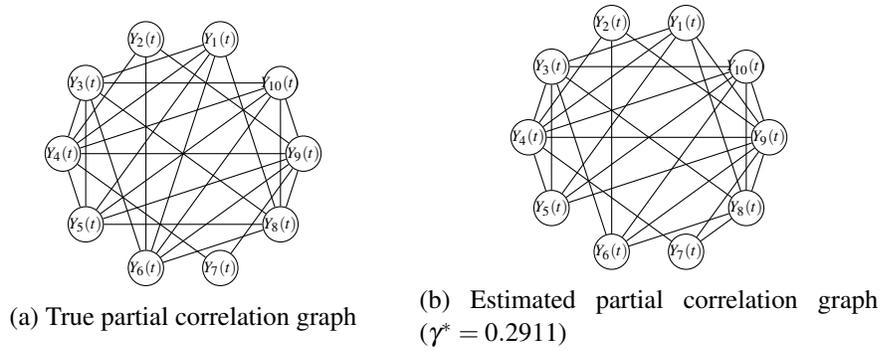
\begin{figure}[h!]
	\centering
	\begin{subfigure}{0.4\textwidth}
		\centering
		\begin{tikzpicture}
		\def \n {10}
		\def \radius {1.6cm}
		\def \margin {8} 
		
		\foreach \s in {1,...,\n}
		{
			\node[shape = circle,inner sep = -9pt,minimum size = 1pt,draw] (\s) at ({360/\n * (\s+1)}:\radius) {{\tiny$Y_{\s}(t)$}};
		}
		
		\draw[] (1) edge (3);
		\draw[] (1) edge (4);
		\draw[] (1) edge (5);
		\draw[] (1) edge (6);
		\draw[] (1) edge (8);
		\draw[] (2) edge (4);
		\draw[] (2) edge (6);
		\draw[] (2) edge (9);
		\draw[] (3) edge (4);
		\draw[] (3) edge (5);
		\draw[] (3) edge (6);
		\draw[] (3) edge (8);
		\draw[] (3) edge (10);
		\draw[] (4) edge (5);
		\draw[] (4) edge (7);
		\draw[] (4) edge (9);
		\draw[] (4) edge (10);
		\draw[] (5) edge (8);
		\draw[] (5) edge (9);
		\draw[] (5) edge (10);
		\draw[] (6) edge (8);
		\draw[] (6) edge (9);
		\draw[] (6) edge (10);
		\draw[] (7) edge (9);
		\draw[] (8) edge (9);
		\draw[] (8) edge (10);
		\draw[] (9) edge (10);
		\end{tikzpicture}
		\caption{True partial correlation graph}
		\label{true_par}
	\end{subfigure}
	\begin{subfigure}{0.4\textwidth} 
		\centering
		\begin{tikzpicture}
		\def \n {10}
		\def \radius {1.6cm}
		\def \margin {2} 
		\foreach \s in {1,...,\n}
		{
			\node[shape = circle,inner sep = -9pt,minimum size = 1pt,draw] (\s) at ({360/\n * (\s+1)}:\radius) {{\tiny$Y_{\s}(t)$}};
		}
		
		\draw[] (1) edge (3);
		\draw[] (1) edge (4);
		\draw[] (1) edge (5);
		\draw[] (1) edge (8);
		\draw[] (1) edge (9);
		\draw[] (2) edge (4);
		\draw[] (2) edge (6);
		\draw[] (2) edge (9);
		\draw[] (3) edge (4);
		\draw[] (3) edge (5);
		\draw[] (3) edge (6);
		\draw[] (3) edge (8);
		\draw[] (3) edge (10);
		\draw[] (4) edge (5);
		\draw[] (4) edge (7);
		\draw[] (4) edge (9);
		\draw[] (4) edge (10);
		\draw[] (5) edge (9);
		\draw[] (5) edge (10);
		\draw[] (6) edge (8);
		\draw[] (6) edge (9);
		\draw[] (6) edge (10);
		\draw[] (7) edge (8);
		\draw[] (7) edge (9);
		\draw[] (8) edge (9);
		\draw[] (8) edge (10);
		\draw[] (9) edge (10);
		\end{tikzpicture}
		\caption{Estimated partial correlation graph ($ \gamma^{*} = 0.2911 $)}
		\label{est_par}
	\end{subfigure}
	\caption{True and estimated partial correlation graphs}
	\label{true_est_par}
\end{figure}

\begin{figure}[h!]
	\centering
	\begin{subfigure}{0.4\textwidth}
		\centering
		\begin{tikzpicture}
		\def \n {10}
		\def \radius {1.6cm}
		\def \margin {8} 
		
		\foreach \s in {1,...,\n}
		{
			\node[shape = circle,inner sep = -9pt,minimum size = 1pt,draw] (\s) at ({360/\n * (\s+1)}:\radius) {{\tiny$Y_{\s}(t)$}};
		}
		
		\draw[->] (1) edge (4);
		\draw[->] (1) edge (6);
		\draw[->] (2) edge (4);
		\draw[->] (2) edge (7);
		\draw[->] (3) edge (1);
		\draw[->] (3) edge (4);
		\draw[->] (3) edge (8);
		\draw[->] (4) edge (1);
		\draw[->] (4) edge (3);
		\draw[->] (4) edge (5);
		\draw[->] (4) edge (9);
		\draw[->] (5) edge (3);
		\draw[->] (5) edge (4);
		\draw[->] (5) edge (10);
		\draw[->] (6) edge (1);
		\draw[->] (6) edge (9);
		\draw[->] (7) edge (2);
		\draw[->] (7) edge (9);
		\draw[->] (8) edge (3);
		\draw[->] (8) edge (6);
		\draw[->] (8) edge (9);
		\draw[->] (9) edge (4);
		\draw[->] (9) edge (6);
		\draw[->] (9) edge (8);
		\draw[->] (9) edge (10);
		\draw[->] (10) edge (5);
		\draw[->] (10) edge (8);
		\draw[->] (10) edge (9);
		\end{tikzpicture}
		\caption{True causality graph}
		\label{true_cas}
	\end{subfigure}
	\begin{subfigure}{0.4\textwidth} 
		\centering
		\begin{tikzpicture}
		\def \n {10}
		\def \radius {1.6cm}
		\def \margin {2} 
		\foreach \s in {1,...,\n}
		{
			\node[shape = circle,inner sep = -9pt,minimum size = 1pt,draw] (\s) at ({360/\n * (\s+1)}:\radius) {{\tiny$Y_{\s}(t)$}};
		}
		
		\draw[->] (1) edge (4);
		\draw[->] (1) edge (6);
		\draw[->] (1) edge (9);
		\draw[->] (2) edge (4);
		\draw[->] (2) edge (7);
		\draw[->] (3) edge (1);
		\draw[->] (3) edge (4);
		\draw[->] (3) edge (8);
		\draw[->] (4) edge (1);
		\draw[->] (4) edge (3);
		\draw[->] (4) edge (5);
		\draw[->] (4) edge (9);
		\draw[->] (5) edge (3);
		\draw[->] (5) edge (4);
		\draw[->] (5) edge (10);
		\draw[->] (6) edge (1);
		\draw[->] (6) edge (3);
		\draw[->] (6) edge (4);
		\draw[->] (6) edge (9);
		\draw[->] (7) edge (2);
		\draw[->] (7) edge (4);
		\draw[->] (7) edge (9);
		\draw[->] (8) edge (3);
		\draw[->] (8) edge (6);
		\draw[->] (8) edge (9);
		\draw[->] (9) edge (4);
		\draw[->] (9) edge (6);
		\draw[->] (9) edge (8);
		\draw[->] (9) edge (10);
		\draw[->] (10) edge (5);
		\draw[->] (10) edge (8);
		\draw[->] (10) edge (9);
		
		\end{tikzpicture}
		\caption{Estimated causality graph ($ \gamma^{*} = 0.2911 $)}
		\label{est_cas}
	\end{subfigure}
	\caption{True and estimated causality graphs}
	\label{true_est_cas}
\end{figure}
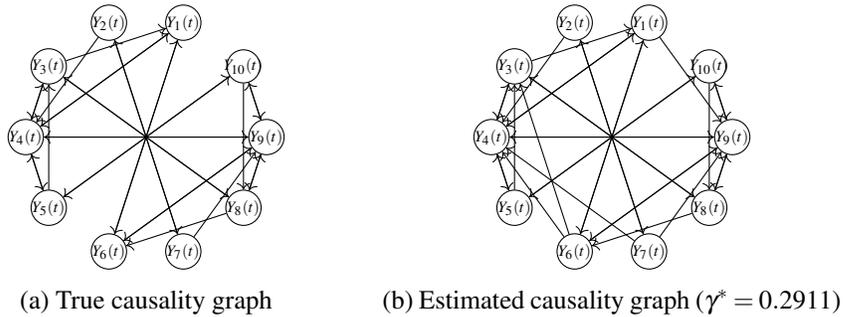


\section{Dengue spread in Greater Mumbai}\label{sec:Real data}

Dengue is a viral infectious disease transmitted to humans through the bites of infected Aedes aegypti mosquitoes. It is one of the most severe health problems being faced globally. The Mumbai Metropolitan Area (previously Greater Mumbai Metropolitan Area) is divided into twenty-four administrative divisions known as wards. The dengue counts are observed weekly from January 2010 to December 2015 from the wards of Greater Mumbai by the public health department of the Municipal Corporation of Greater Mumbai (MCGM). The length of each time series is 318.  

\subsection{Model Estimation}
In this section, we estimate the partial correlation and causality graph using the proposed Algorithm \ref{MCEM_algorithm} from the observed 24-dimensional dengue data set. The time series plot of normalized dengue counts of some of the wards (A, B, and C) are shown in Fig. \ref{data}. 

\begin{remark}
\label{remark_MCGM}
In the plots below, the absolute counts are not shown due to data privacy rules of MCGM. However the proposed algorithm and the final results are based on the actual counts.
\end{remark}

From these time series plots, the increasing trend and yearly seasonality with high dengue counts being reported during monsoon months are observed. The non-stationary behaviour of the dengue counts is also apparent from these plots. Thus, a linear trend and yearly seasonality represented by a pair of sine and cosine terms, are considered in the covariate vector $\textbf{z}_{t,i}$ for all wards $i = 1,\dots,24$ of the model \eqref{A1}. Then model \eqref{A1} becomes 
\begin{align}\label{real_model}
Y_{i}(t)|X_{i}(t),\textbf{z}_{t,i}\sim \text{Poisson}(\mu_{i}(t)).
\end{align}
where $\textbf{z}_{t,i}=\begin{bmatrix}1&t&cos(2\pi t/52)&sin(2\pi t/52)\end{bmatrix}^T$.
\begin{figure}[h!]
	\centering
	\includegraphics[scale=0.25,center]{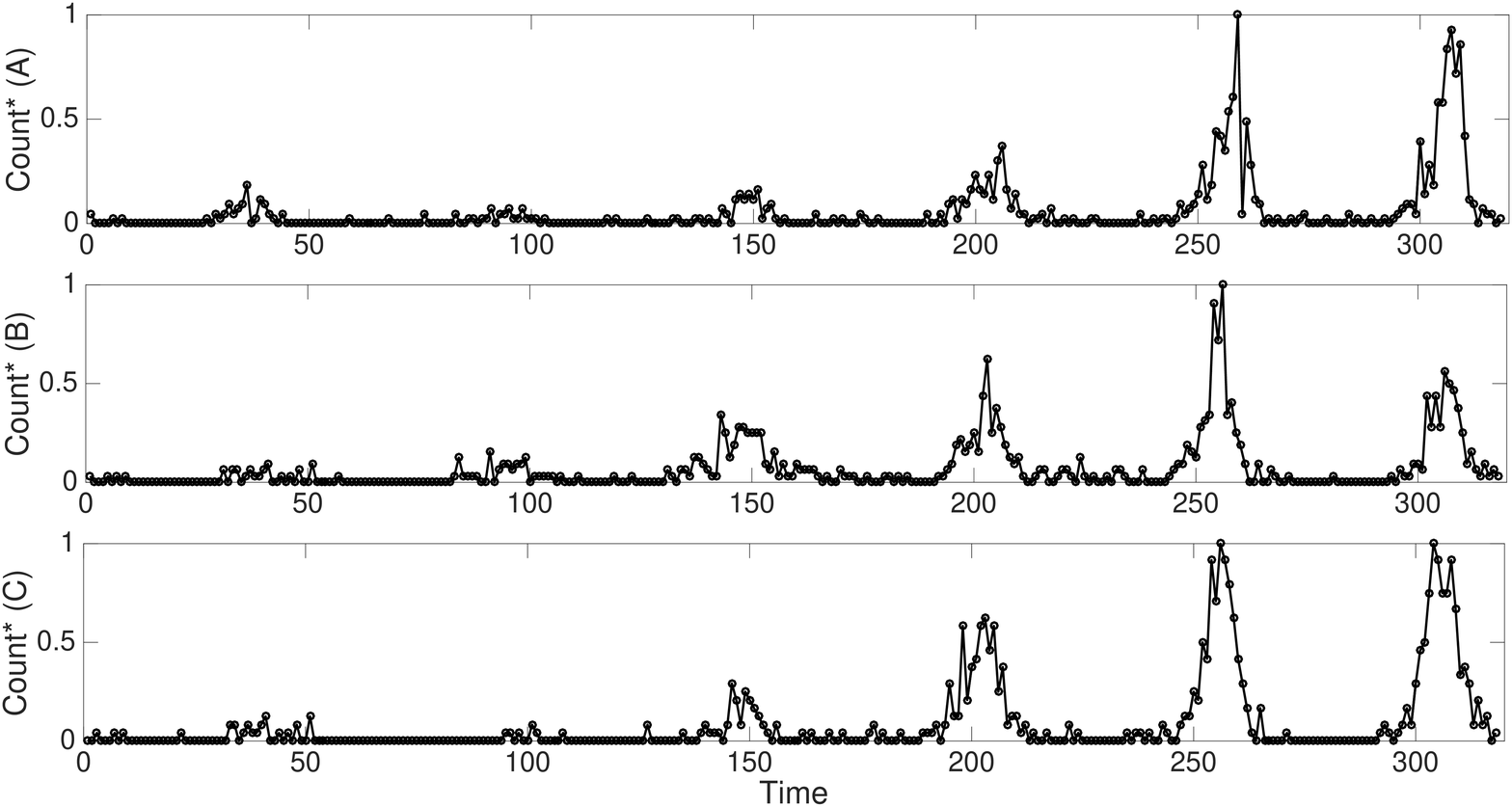}
	\caption{Time series plot of normalized dengue counts of A, B and C wards where count* indicates the true count normalized with respect to the maximum weekly count observed over the observed duration (318 weeks)}
	\label{data}
\end{figure}

We estimate the unknown parameters of the model \eqref{real_model} and \eqref{A3} with order ranging from $p=0$ to $p=3$ using Algorithm \ref{MCEM_algorithm}. All these models are ranked with second-order Akaike (AIC$_\text{c}$) and Bayes information criteria (BIC) which are given in Table. \ref{AIC_BIC}. The model of order $p=1$ which corresponds to the lowest AIC$_\text{c}$ and BIC score is selected.  
\begin{table}[h!]
	\centering
	\begin{tabular}{|c|c|c|}
		\hline
		& AIC$_\text{c}$ & BIC    \\ \hline
		$p=0$ & 8924.5  & 13195.8 \\ \hline
		\bs{$p=1$} & \textbf{7176.2}  & \textbf{12851.1} \\ \hline
		$p=2$ & 7377.6  & 16530.3 \\ \hline
		$p=3$ & 7405.1  & 19844.9 \\ \hline
	\end{tabular}
	\caption{AIC$_\text{c}$ and BIC scores for different values of order $p$}
	\label{AIC_BIC}
\end{table}
Thus for the model with order $p=1$, we find the trade-off curve between the conditional log-likelihood ($l(\mathbf{Y}, \mathbf{X}, \bs{\theta})$) and the $\ell_{1}$-type regularization function ($ h_{1}(W_{0},W_{1},\dots,W_{p}) $) which is shown in Fig. \ref{Tradeoff}. After thresholding (see section \ref{threshold}) with threshold $\rho^{*}=0.1$, the normalized inverse spectral density matrices from \eqref{B53} along the trade-off curve (Fig. \ref{Tradeoff}) are given in Fig. \ref{Graphs}. We observe that sparsity pattern varies in the estimated inverse spectral density matrix from dense ($\gamma $ small) to diagonal ($\gamma $ large). We rank all of the partial correlation graphs along the trade-off curve with BIC scores which is shown in Fig. \ref{lambda_BIC}. The regularization parameter ($\gamma$) values along the trade-off curve and the corresponding BIC scores are given in Table. \ref{Lam_BIC}. The estimated partial correlations and causalities between multiple count time series corresponding to $\gamma^{*}  = 0.682$ are shown in separate graphs given in Fig.\ref{estimated_par} and Fig. \ref{estimated_cas}.
\\

\begin{figure}[]
	\centering
	\includegraphics[scale=0.18,center]{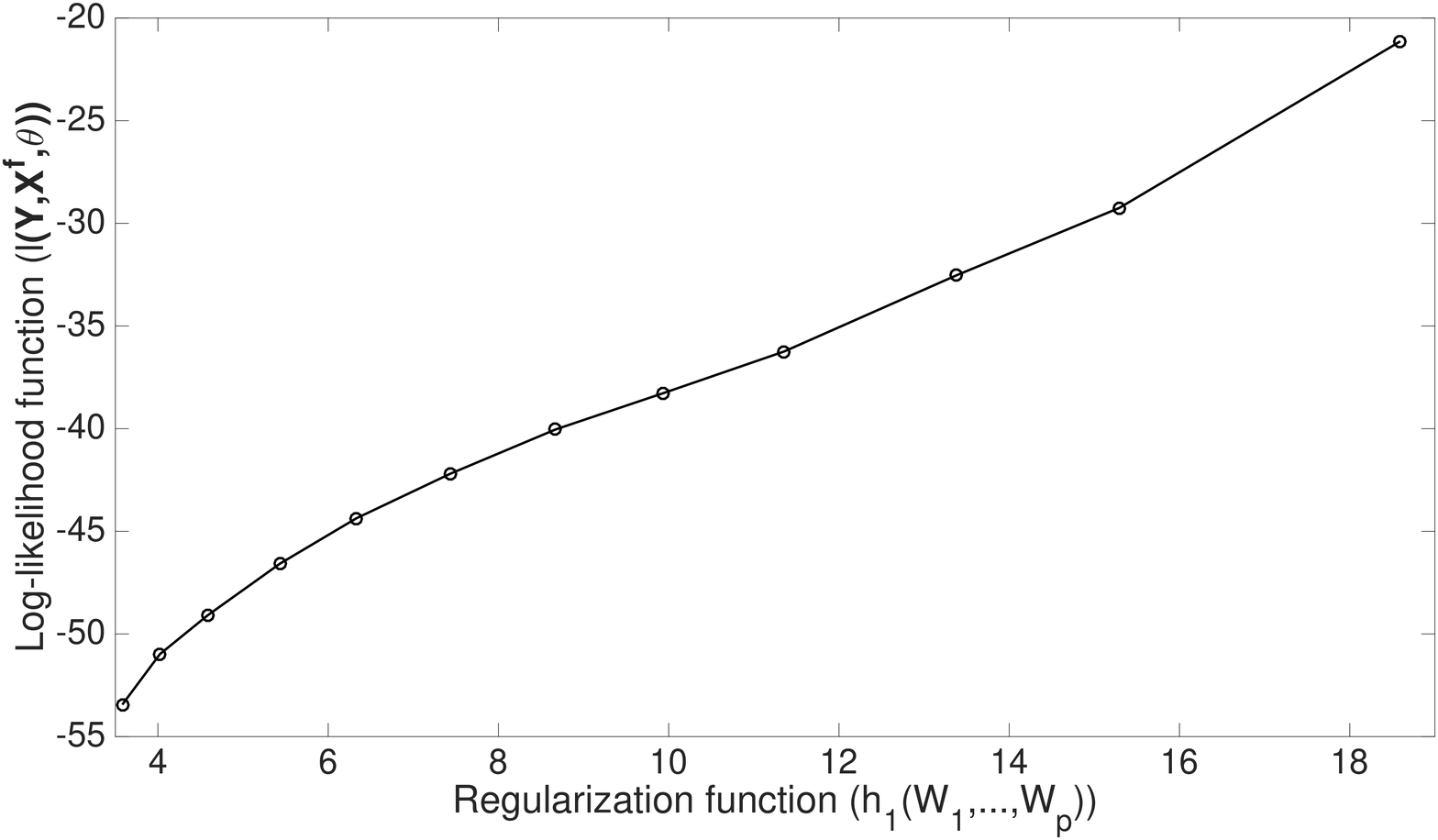}
	\caption{Trade-off curve between log-likelihood function ($l(\mathbf{Y}, \mathbf{X}, \bs{\theta})$) and regularization function ($h_{1}(W_{0},W_{1},\hdots, W_{p})$)}
	\label{Tradeoff}
\end{figure}
\vspace{-0.5cm}
\begin{figure}[]
	\centering
	\includegraphics[scale=0.4,center]{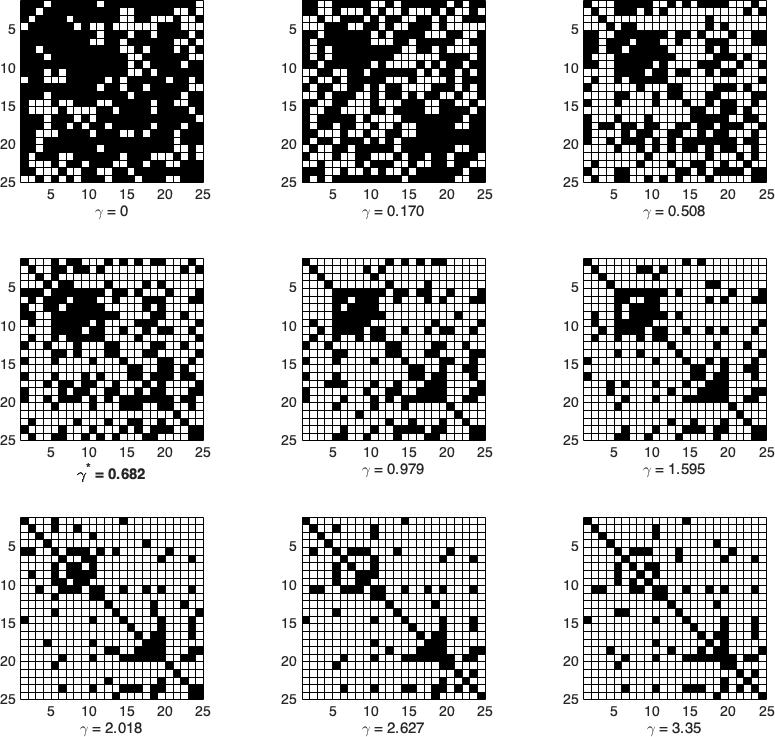}
	\caption{Partial correlation graphs along the trade-off curve in Fig. \ref{Tradeoff} \ref{Tradeoff}}
	\label{Graphs}
\end{figure}
\begin{figure}[]
	\centering
	\includegraphics[scale=0.19,center]{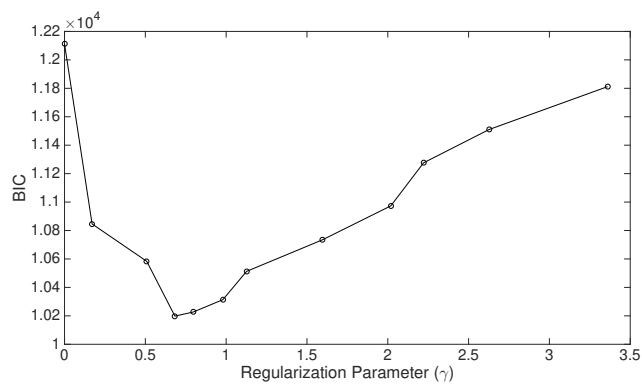}
	\caption{BIC scores along the trade-off curve in Fig. \ref{Tradeoff}}
	\label{lambda_BIC}
\end{figure}
\begin{table}[]
	\centering
	\begin{tabular}{|c|c|c|c|c|c|}
		\hline
		$\gamma$ & BIC      & $\gamma$ & BIC      & $\gamma$ & BIC      \\ \hline
		0        & 12111.43 & 0.7969   & 10227.05 & 2.0186   & 10973.20 \\ \hline
		0.1706   & 10846.09 & 0.9795   & 10313.73 & 2.2213   & 11275.20 \\ \hline
		0.5084   & 10583.76 & 1.1266   & 10512.76 & 2.6271   & 11510.35 \\ \hline
		\textbf{0.6829}   & \textbf{10197.97} & 1.5951   & 10734.25 & 3.35     & 11811.79 \\ \hline
	\end{tabular}
	\caption{Regularization parameter ($\gamma$) values along the trade-off curve in Fig. \ref{Tradeoff} and corresponding BIC scores}
	\label{Lam_BIC}
\end{table}
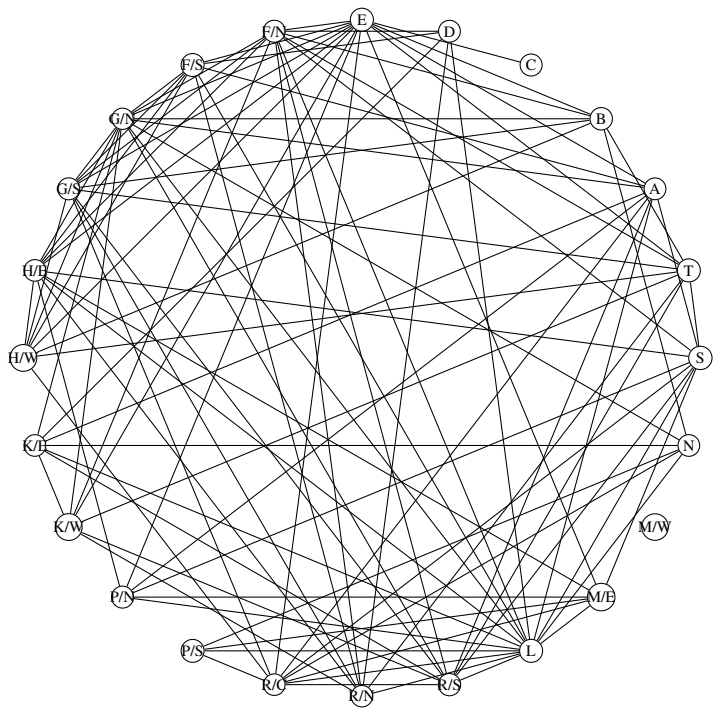
\begin{figure}[]
	\centering
	\begin{tikzpicture}
	\def \n {24}
	\def \radius {4.5cm}
	\def \margin {2} 
	\def \a {6}
	\def \b {1}
	\def \c {5}
	\def \d {7}
	
	\node[shape = circle,inner sep = -\c pt,minimum size = \b pt,draw] (1) at ({360/\n * (1+1)}:\radius){\tiny A};
	\node[shape = circle,inner sep = -\c pt,minimum size = \b pt,draw] (2) at ({360/\n * (2+1)}:\radius){{\tiny B}};
	\node[shape = circle,inner sep = -\c pt,minimum size = \b pt,draw] (3) at ({360/\n * (3+1)}:\radius){{\tiny C}};
	\node[shape = circle,inner sep = -\c pt,minimum size = \b pt,draw] (4) at ({360/\n * (4+1)}:\radius){{\tiny D}};
	\node[shape = circle,inner sep = -\c pt,minimum size = \b pt,draw] (5) at ({360/\n * (5+1)}:\radius){{\tiny E}};
	\node[shape = circle,inner sep = -\a pt,minimum size = \b pt,draw] (6) at ({360/\n * (6+1)}:\radius){{\tiny F/N}};
	\node[shape = circle,inner sep = -\a pt,minimum size = \b pt,draw] (7) at ({360/\n * (7+1)}:\radius){{\tiny F/S}};
	\node[shape = circle,inner sep = -\a pt,minimum size = \b pt,draw] (8) at ({360/\n * (8+1)}:\radius){{\tiny G/N}};
	\node[shape = circle,inner sep = -\a pt,minimum size = \b pt,draw] (9) at ({360/\n * (9+1)}:\radius){{\tiny G/S}};
	\node[shape = circle,inner sep = -\a pt,minimum size = \b pt,draw] (10) at ({360/\n * (10+1)}:\radius){{\tiny H/E}};
	\node[shape = circle,inner sep = -\d pt,minimum size = \b pt,draw] (11) at ({360/\n * (11+1)}:\radius){{\tiny H/W}};
	\node[shape = circle,inner sep = -\a pt,minimum size = \b pt,draw] (12) at ({360/\n * (12+1)}:\radius){{\tiny K/E}};
	\node[shape = circle,inner sep = -\d pt,minimum size = \b pt,draw] (13) at ({360/\n * (13+1)}:\radius){{\tiny K/W}};
	\node[shape = circle,inner sep = -\a pt,minimum size = \b pt,draw] (14) at ({360/\n * (14+1)}:\radius){{\tiny P/N}};
	\node[shape = circle,inner sep = -\a pt,minimum size = \b pt,draw] (15) at ({360/\n * (15+1)}:\radius){{\tiny P/S}};
	\node[shape = circle,inner sep = -\a pt,minimum size = \b pt,draw] (16) at ({360/\n * (16+1)}:\radius){{\tiny R/C}};
	\node[shape = circle,inner sep = -\a pt,minimum size = \b pt,draw] (17) at ({360/\n * (17+1)}:\radius){{\tiny R/N}};
	\node[shape = circle,inner sep = -\a pt,minimum size = \b pt,draw] (18) at ({360/\n * (18+1)}:\radius){{\tiny R/S}};
	\node[shape = circle,inner sep = -\c pt,minimum size = \b pt,draw] (19) at ({360/\n * (19+1)}:\radius){{\tiny L}};
	\node[shape = circle,inner sep = -\d pt,minimum size = \b pt,draw] (20) at ({360/\n * (20+1)}:\radius){{\tiny M/E}};
	\node[shape = circle,inner sep = -\d pt,minimum size = \b pt,draw] (21) at ({360/\n * (21+1)}:\radius){{\tiny M/W}};
	\node[shape = circle,inner sep = -\c pt,minimum size = \b pt,draw] (22) at ({360/\n * (22+1)}:\radius){{\tiny N}};
	\node[shape = circle,inner sep = -\c pt,minimum size = \b pt,draw] (23) at ({360/\n * (23+1)}:\radius){{\tiny S}};
	\node[shape = circle,inner sep = -\c pt,minimum size = \b pt,draw] (24) at ({360/\n * (24+1)}:\radius){{\tiny T}};

	\draw(1) edge (5);
	\draw(1) edge (7);
	\draw(1) edge (8);
	\draw(1) edge (12);
	\draw(1) edge (14);
	\draw(1) edge (16);
	\draw(1) edge (18);
	\draw(1) edge (19);
	\draw(1) edge (23);
	\draw(2) edge (5);
	\draw(2) edge (6);
	\draw(2) edge (8);
	\draw(2) edge (9);
	\draw(2) edge (11);
	\draw(2) edge (22);
	\draw(2) edge (24);
	\draw(3) edge (5);
	\draw(4) edge (6);
	\draw(4) edge (7);
	\draw(4) edge (12);
	\draw(4) edge (17);
	\draw(4) edge (19);
	\draw(5) edge (6);
	\draw(5) edge (7);
	\draw(5) edge (8);
	\draw(5) edge (9);
	\draw(5) edge (10);
	\draw(5) edge (11);
	\draw(5) edge (13);
	\draw(5) edge (14);
	\draw(5) edge (16);
	\draw(5) edge (20);
	\draw(5) edge (24);
	\draw(6) edge (8);
	\draw(6) edge (9);
	\draw(6) edge (10);
	\draw(6) edge (13);
	\draw(6) edge (17);
	\draw(6) edge (18);
	\draw(6) edge (19);
	\draw(6) edge (23);
	\draw(6) edge (24);
	\draw(7) edge (8);
	\draw(7) edge (9);
	\draw(7) edge (10);
	\draw(7) edge (11);
	\draw(7) edge (17);
	\draw(7) edge (19);
	\draw(8) edge (9);
	\draw(8) edge (10);
	\draw(8) edge (11);
	\draw(8) edge (12);
	\draw(8) edge (13);
	\draw(8) edge (18);
	\draw(8) edge (19);
	\draw(8) edge (22);
	\draw(9) edge (11);
	\draw(9) edge (16);
	\draw(9) edge (17);
	\draw(9) edge (18);
	\draw(9) edge (24);
	\draw(10) edge (11);
	\draw(10) edge (14);
	\draw(10) edge (17);
	\draw(10) edge (19);
	\draw(10) edge (20);
	\draw(10) edge (23);
	\draw(11) edge (16);
	\draw(11) edge (24);
	\draw(12) edge (13);
	\draw(12) edge (18);
	\draw(12) edge (19);
	\draw(12) edge (22);
	\draw(13) edge (17);
	\draw(13) edge (18);
	\draw(13) edge (24);
	\draw(14) edge (19);
	\draw(14) edge (20);
	\draw(14) edge (23);
	\draw(15) edge (16);
	\draw(15) edge (19);
	\draw(15) edge (20);
	\draw(15) edge (22);
	\draw(16) edge (18);
	\draw(16) edge (19);
	\draw(16) edge (20);
	\draw(16) edge (22);
	\draw(16) edge (23);
	\draw(17) edge (19);
	\draw(17) edge (24);
	\draw(18) edge (19);
	\draw(18) edge (23);
	\draw(18) edge (24);
	\draw(19) edge (20);
	\draw(19) edge (22);
	\draw(19) edge (23);
	\draw(20) edge (23);
	\draw(23) edge (24);		
	\end{tikzpicture}
	\caption{Estimated partial correlation graph with regularization ($\gamma^* = 0.682$)}
	\label{estimated_par}
\end{figure}

\begin{figure}[]
	\centering
	\begin{tikzpicture}
	\def \n {24}
	\def \radius {4.5cm}
	\def \margin {2} 
	\def \a {6}
	\def \b {1}
	\def \c {5}
	\def \d {7}
	
	\node[shape = circle,inner sep = -\c pt,minimum size = \b pt,draw] (1) at ({360/\n * (1+1)}:\radius){\tiny A};
	\node[shape = circle,inner sep = -\c pt,minimum size = \b pt,draw] (2) at ({360/\n * (2+1)}:\radius){{\tiny B}};
	\node[shape = circle,inner sep = -\c pt,minimum size = \b pt,draw] (3) at ({360/\n * (3+1)}:\radius){{\tiny C}};
	\node[shape = circle,inner sep = -\c pt,minimum size = \b pt,draw] (4) at ({360/\n * (4+1)}:\radius){{\tiny D}};
	\node[shape = circle,inner sep = -\c pt,minimum size = \b pt,draw] (5) at ({360/\n * (5+1)}:\radius){{\tiny E}};
	\node[shape = circle,inner sep = -\a pt,minimum size = \b pt,draw] (6) at ({360/\n * (6+1)}:\radius){{\tiny F/N}};
	\node[shape = circle,inner sep = -\a pt,minimum size = \b pt,draw] (7) at ({360/\n * (7+1)}:\radius){{\tiny F/S}};
	\node[shape = circle,inner sep = -\a pt,minimum size = \b pt,draw] (8) at ({360/\n * (8+1)}:\radius){{\tiny G/N}};
	\node[shape = circle,inner sep = -\a pt,minimum size = \b pt,draw] (9) at ({360/\n * (9+1)}:\radius){{\tiny G/S}};
	\node[shape = circle,inner sep = -\a pt,minimum size = \b pt,draw] (10) at ({360/\n * (10+1)}:\radius){{\tiny H/E}};
	\node[shape = circle,inner sep = -\d pt,minimum size = \b pt,draw] (11) at ({360/\n * (11+1)}:\radius){{\tiny H/W}};
	\node[shape = circle,inner sep = -\a pt,minimum size = \b pt,draw] (12) at ({360/\n * (12+1)}:\radius){{\tiny K/E}};
	\node[shape = circle,inner sep = -\d pt,minimum size = \b pt,draw] (13) at ({360/\n * (13+1)}:\radius){{\tiny K/W}};
	\node[shape = circle,inner sep = -\a pt,minimum size = \b pt,draw] (14) at ({360/\n * (14+1)}:\radius){{\tiny P/N}};
	\node[shape = circle,inner sep = -\a pt,minimum size = \b pt,draw] (15) at ({360/\n * (15+1)}:\radius){{\tiny P/S}};
	\node[shape = circle,inner sep = -\a pt,minimum size = \b pt,draw] (16) at ({360/\n * (16+1)}:\radius){{\tiny R/C}};
	\node[shape = circle,inner sep = -\a pt,minimum size = \b pt,draw] (17) at ({360/\n * (17+1)}:\radius){{\tiny R/N}};
	\node[shape = circle,inner sep = -\a pt,minimum size = \b pt,draw] (18) at ({360/\n * (18+1)}:\radius){{\tiny R/S}};
	\node[shape = circle,inner sep = -\c pt,minimum size = \b pt,draw] (19) at ({360/\n * (19+1)}:\radius){{\tiny L}};
	\node[shape = circle,inner sep = -\d pt,minimum size = \b pt,draw] (20) at ({360/\n * (20+1)}:\radius){{\tiny M/E}};
	\node[shape = circle,inner sep = -\d pt,minimum size = \b pt,draw] (21) at ({360/\n * (21+1)}:\radius){{\tiny M/W}};
	\node[shape = circle,inner sep = -\c pt,minimum size = \b pt,draw] (22) at ({360/\n * (22+1)}:\radius){{\tiny N}};
	\node[shape = circle,inner sep = -\c pt,minimum size = \b pt,draw] (23) at ({360/\n * (23+1)}:\radius){{\tiny S}};
	\node[shape = circle,inner sep = -\c pt,minimum size = \b pt,draw] (24) at ({360/\n * (24+1)}:\radius){{\tiny T}};
	
	\draw[->] (1) edge (5);
	\draw[->] (1) edge (7);
	\draw[->] (1) edge (12);
	\draw[->] (1) edge (13);
	\draw[->] (1) edge (14);
	\draw[->] (1) edge (16);
	\draw[->] (1) edge (19);
	\draw[->] (2) edge (1);
	\draw[->] (2) edge (5);
	\draw[->] (2) edge (6);
	\draw[->] (2) edge (7);
	\draw[->] (2) edge (8);
	\draw[->] (2) edge (9);
	\draw[->] (2) edge (11);
	\draw[->] (2) edge (14);
	\draw[->] (2) edge (20);
	\draw[->] (3) edge (5);
	\draw[->] (3) edge (10);
	\draw[->] (3) edge (13);
	\draw[->] (4) edge (3);
	\draw[->] (4) edge (6);
	\draw[->] (4) edge (7);
	\draw[->] (4) edge (12);
	\draw[->] (4) edge (17);
	\draw[->] (4) edge (18);
	\draw[->] (4) edge (19);
	\draw[->] (5) edge (1);
	\draw[->] (5) edge (2);
	\draw[->] (5) edge (4);
	\draw[->] (5) edge (6);
	\draw[->] (5) edge (7);
	\draw[->] (5) edge (8);
	\draw[->] (5) edge (12);
	\draw[->] (5) edge (13);
	\draw[->] (5) edge (14);
	\draw[->] (5) edge (17);
	\draw[->] (5) edge (18);
	\draw[->] (5) edge (20);
	\draw[->] (5) edge (23);
	\draw[->] (6) edge (1);
	\draw[->] (6) edge (2);
	\draw[->] (6) edge (3);
	\draw[->] (6) edge (4);
	\draw[->] (6) edge (5);
	\draw[->] (6) edge (7);
	\draw[->] (6) edge (8);
	\draw[->] (6) edge (9);
	\draw[->] (6) edge (10);
	\draw[->] (6) edge (15);
	\draw[->] (6) edge (16);
	\draw[->] (6) edge (19);
	\draw[->] (6) edge (20);
	\draw[->] (6) edge (22);
	\draw[->] (6) edge (23);
	\draw[->] (6) edge (24);
	\draw[->] (7) edge (5);
	\draw[->] (7) edge (6);
	\draw[->] (7) edge (8);
	\draw[->] (7) edge (9);
	\draw[->] (7) edge (10);
	\draw[->] (7) edge (11);
	\draw[->] (7) edge (13);
	\draw[->] (7) edge (15);
	\draw[->] (7) edge (17);
	\draw[->] (7) edge (18);
	\draw[->] (7) edge (19);
	\draw[->] (8) edge (1);
	\draw[->] (8) edge (2);
	\draw[->] (8) edge (3);
	\draw[->] (8) edge (4);
	\draw[->] (8) edge (5);
	\draw[->] (8) edge (7);
	\draw[->] (8) edge (9);
	\draw[->] (8) edge (10);
	\draw[->] (8) edge (12);
	\draw[->] (8) edge (17);
	\draw[->] (8) edge (18);
	\draw[->] (8) edge (19);
	\draw[->] (8) edge (20);
	\draw[->] (8) edge (21);
	\draw[->] (8) edge (22);
	\draw[->] (9) edge (1);
	\draw[->] (9) edge (2);
	\draw[->] (9) edge (5);
	\draw[->] (9) edge (7);
	\draw[->] (9) edge (8);
	\draw[->] (9) edge (13);
	\draw[->] (9) edge (16);
	\draw[->] (9) edge (18);
	\draw[->] (9) edge (19);
	\draw[->] (9) edge (21);
	\draw[->] (9) edge (24);
	\draw[->] (10) edge (1);
	\draw[->] (10) edge (3);
	\draw[->] (10) edge (4);
	\draw[->] (10) edge (5);
	\draw[->] (10) edge (7);
	\draw[->] (10) edge (8);
	\draw[->] (10) edge (11);
	\draw[->] (10) edge (12);
	\draw[->] (10) edge (13);
	\draw[->] (10) edge (14);
	\draw[->] (10) edge (16);
	\draw[->] (10) edge (17);
	\draw[->] (10) edge (18);
	\draw[->] (10) edge (19);
	\draw[->] (10) edge (21);
	\draw[->] (10) edge (22);
	\draw[->] (10) edge (24);
	\draw[->] (11) edge (2);
	\draw[->] (11) edge (8);
	\draw[->] (11) edge (10);
	\draw[->] (11) edge (18);
	\draw[->] (11) edge (20);
	\draw[->] (11) edge (22);
	\draw[->] (11) edge (23);
	\draw[->] (12) edge (1);
	\draw[->] (12) edge (2);
	\draw[->] (12) edge (9);
	\draw[->] (12) edge (13);
	\draw[->] (12) edge (17);
	\draw[->] (12) edge (18);
	\draw[->] (12) edge (19);
	\draw[->] (12) edge (20);
	\draw[->] (12) edge (22);
	\draw[->] (13) edge (8);
	\draw[->] (13) edge (11);
	\draw[->] (13) edge (17);
	\draw[->] (13) edge (18);
	\draw[->] (14) edge (1);
	\draw[->] (14) edge (2);
	\draw[->] (14) edge (3);
	\draw[->] (14) edge (5);
	\draw[->] (14) edge (10);
	\draw[->] (14) edge (13);
	\draw[->] (14) edge (15);
	\draw[->] (14) edge (19);
	\draw[->] (14) edge (20);
	\draw[->] (14) edge (23);
	\draw[->] (14) edge (24);
	\draw[->] (15) edge (4);
	\draw[->] (15) edge (5);
	\draw[->] (15) edge (13);
	\draw[->] (15) edge (16);
	\draw[->] (15) edge (19);
	\draw[->] (15) edge (20);
	\draw[->] (15) edge (22);
	\draw[->] (15) edge (24);
	\draw[->] (16) edge (1);
	\draw[->] (16) edge (3);
	\draw[->] (16) edge (6);
	\draw[->] (16) edge (9);
	\draw[->] (16) edge (11);
	\draw[->] (16) edge (17);
	\draw[->] (16) edge (18);
	\draw[->] (16) edge (19);
	\draw[->] (16) edge (20);
	\draw[->] (16) edge (21);
	\draw[->] (16) edge (22);
	\draw[->] (16) edge (23);
	\draw[->] (17) edge (1);
	\draw[->] (17) edge (3);
	\draw[->] (17) edge (6);
	\draw[->] (17) edge (10);
	\draw[->] (17) edge (11);
	\draw[->] (17) edge (12);
	\draw[->] (17) edge (13);
	\draw[->] (17) edge (15);
	\draw[->] (17) edge (19);
	\draw[->] (17) edge (20);
	\draw[->] (17) edge (21);
	\draw[->] (17) edge (24);
	\draw[->] (18) edge (2);
	\draw[->] (18) edge (6);
	\draw[->] (18) edge (7);
	\draw[->] (18) edge (8);
	\draw[->] (18) edge (9);
	\draw[->] (18) edge (12);
	\draw[->] (18) edge (16);
	\draw[->] (18) edge (17);
	\draw[->] (18) edge (19);
	\draw[->] (19) edge (1);
	\draw[->] (19) edge (7);
	\draw[->] (19) edge (9);
	\draw[->] (19) edge (12);
	\draw[->] (19) edge (13);
	\draw[->] (19) edge (15);
	\draw[->] (19) edge (16);
	\draw[->] (19) edge (17);
	\draw[->] (19) edge (18);
	\draw[->] (19) edge (23);
	\draw[->] (19) edge (24);
	\draw[->] (20) edge (1);
	\draw[->] (20) edge (5);
	\draw[->] (20) edge (8);
	\draw[->] (20) edge (12);
	\draw[->] (20) edge (19);
	\draw[->] (21) edge (18);
	\draw[->] (22) edge (1);
	\draw[->] (22) edge (2);
	\draw[->] (22) edge (3);
	\draw[->] (22) edge (8);
	\draw[->] (22) edge (12);
	\draw[->] (22) edge (15);
	\draw[->] (22) edge (16);
	\draw[->] (22) edge (19);
	\draw[->] (22) edge (20);
	\draw[->] (22) edge (23);
	\draw[->] (23) edge (1);
	\draw[->] (23) edge (2);
	\draw[->] (23) edge (6);
	\draw[->] (23) edge (9);
	\draw[->] (23) edge (10);
	\draw[->] (23) edge (12);
	\draw[->] (23) edge (13);
	\draw[->] (23) edge (14);
	\draw[->] (23) edge (15);
	\draw[->] (23) edge (17);
	\draw[->] (23) edge (18);
	\draw[->] (23) edge (19);
	\draw[->] (23) edge (20);
	\draw[->] (23) edge (21);
	\draw[->] (23) edge (24);
	\draw[->] (24) edge (2);
	\draw[->] (24) edge (5);
	\draw[->] (24) edge (6);
	\draw[->] (24) edge (8);
	\draw[->] (24) edge (9);
	\draw[->] (24) edge (11);
	\draw[->] (24) edge (13);
	\draw[->] (24) edge (17);
	\draw[->] (24) edge (18);
	\draw[->] (24) edge (19);
	\draw[->] (24) edge (22);
	\draw[->] (24) edge (23);
	
	\end{tikzpicture}
	\caption{Estimated causality graph}
	\label{estimated_cas}
\end{figure}
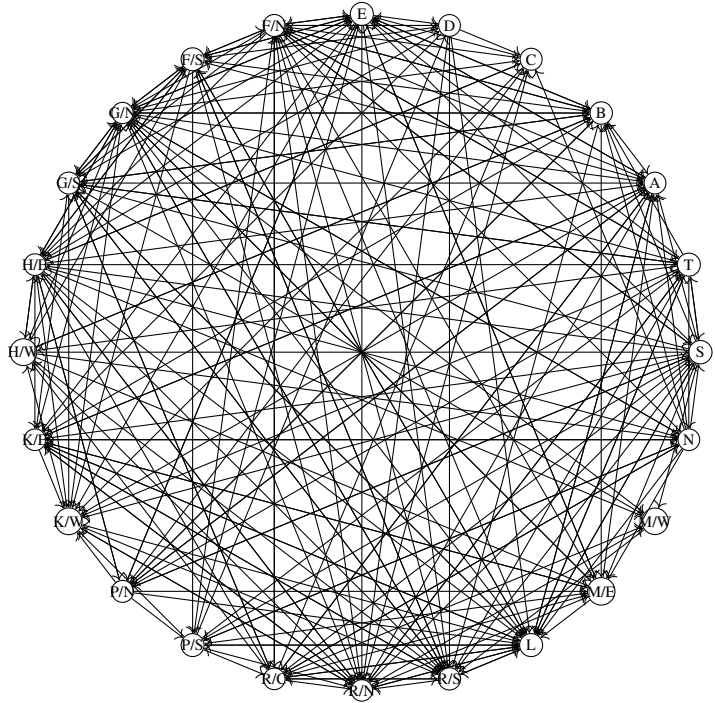

\subsection{Inference}
Based on the number of undirected edges between the wards in the estimated partial correlation graph in Fig. \ref{estimated_par}, we construct a colour coded map (see  Fig. \ref{Map1}) showing the number of incident edges for each ward. 
Similarly, based on the number of outward edges of each ward in the estimated causality graph of Fig. \ref{estimated_cas}, we construct a colour coded map (see Fig. \ref{Map2}) showing the number of outgoing edges for each ward. All estimated model parameters such as undirected and directed edges for each ward along with normalized disease count (see Remark \ref{remark_MCGM}), are presented in the Table \ref{ward_count} . 

\begin{figure}[]
	\centering
	\includegraphics[scale=0.8,center]{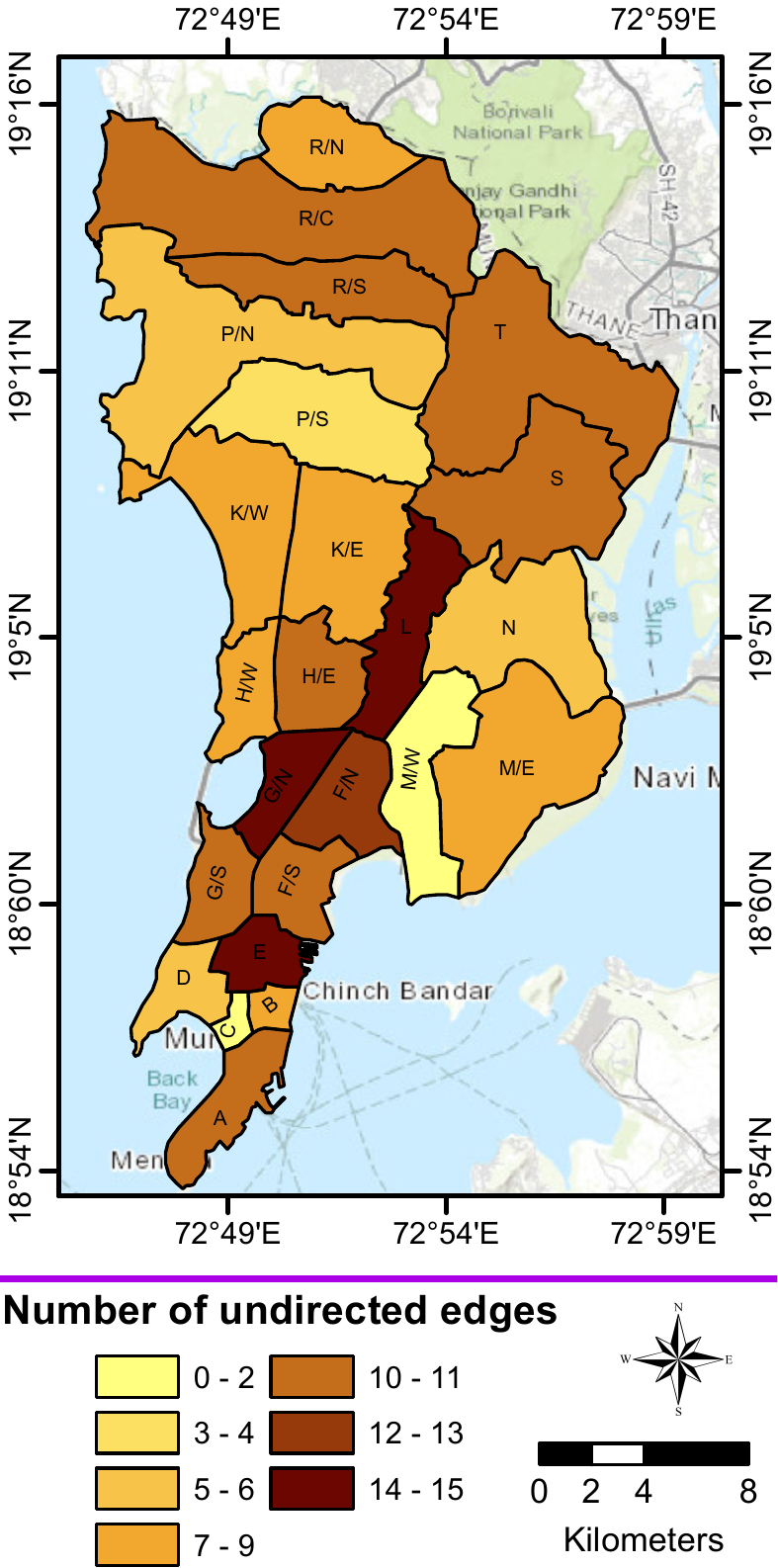}
	\caption{Each ward of Greater Mumbai is colour coded with the number of undirected edges in the estimated partial correlation graph Fig. \ref{estimated_par}}
	\label{Map1}
\end{figure}
\begin{figure}[]
	\centering
	\includegraphics[scale=0.8,center]{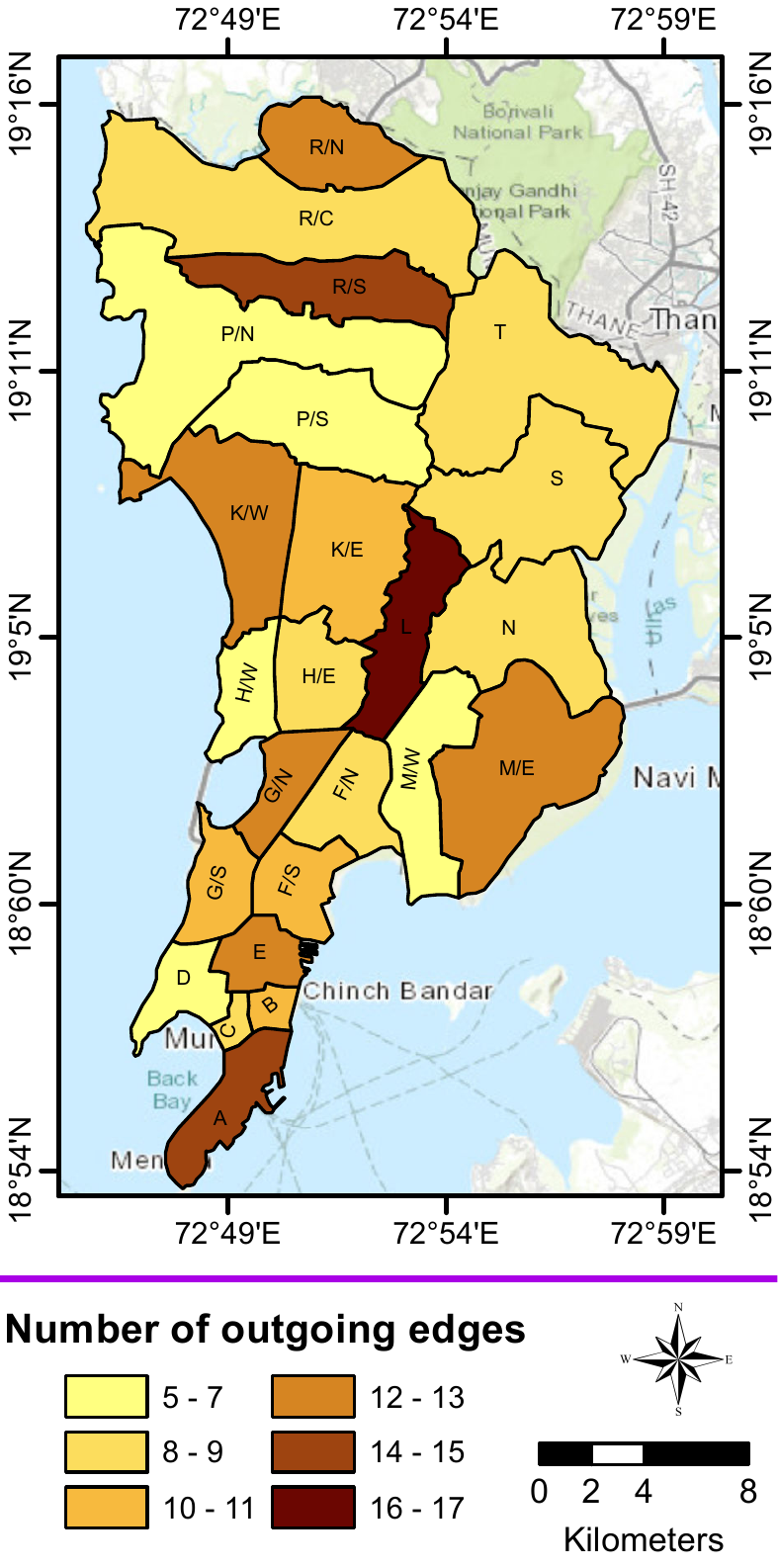}
	\caption{Each ward of Greater Mumbai is colour coded with the number of outgoing edges in the estimated causality graph Fig. \ref{estimated_cas}}
	\label{Map2}
\end{figure}
The plots of the normalized dengue counts, number of undirected edges from the estimated partial correlation graph Fig. \ref{estimated_par} and number of incoming and outgoing edges from the estimated causality graph Fig. \ref{estimated_cas} of each ward, are shown in Fig. \ref{count_edge_inc_out}. In Fig .\ref{count_edge_inc_out}, all counts are normalized with their corresponding maximum. 
\begin{table}[]
	\centering
	\begin{tabular}{|c|c|c|c|c|c|c|}
		\hline
		Ward & Count & UE & IE & OE & $IW$ & $OW$ \\ \hline
		A &0.252&0.6&0.411&0.823&0.508&0.631 \\ \hline
		B &0.217&0.466&0.529&0.647&0.369&0.486 \\ \hline
		C &0.201&0.066&0.176&0.470&0.202&0.295 \\ \hline
		D &0.278&0.333&0.411&0.294&0.343&0.222 \\ \hline
		E &0.984&0.933&0.764&0.705&0.799&0.429 \\ \hline
		F/N &0.668&0.8&0.941&0.529&0.885&0.307 \\ \hline
		F/S &0.481&0.6&0.647&0.588&0.624&0.430 \\ \hline
		G/N &0.586&0.866&0.882&0.705&0.937&0.567 \\ \hline
		G/S &0.964&0.666&0.647&0.588&0.727&0.317 \\ \hline
		H/E &1&0.666&1&0.470&1&0.430 \\ \hline
		H/W &0.370&0.533&0.411&0.411&0.421&0.354 \\ \hline
		K/E &0.365&0.466&0.529&0.647&0.584&0.414 \\ \hline
		K/W &0.372&0.466&0.235&0.764&0.358&0.477 \\ \hline
		P/N &0.607&0.4&0.647&0.294&0.718&0.287 \\ \hline
		P/S &0.431&0.266&0.470&0.411&0.510&0.276 \\ \hline
		R/C &0.352&0.666&0.705&0.470&0.695&0.347 \\ \hline
		R/N &0.638&0.533&0.705&0.705&0.753&0.425 \\ \hline
		R/S &0.521&0.666&0.529&0.823&0.639&0.706 \\ \hline
		L &0.522&1&0.647&1&0.766&1 \\ \hline
		M/E &0.261&0.466&0.294&0.705&0.318&0.515 \\ \hline
		M/W &0.116&0&0.058&0.352&0.143&0.419 \\ \hline
		N &0.379&0.4&0.588&0.470&0.602&0.308 \\ \hline
		S &0.608&0.6&0.882&0.470&0.862&0.362 \\ \hline
		T &0.272&0.6&0.705&0.470&0.573&0.242 \\ \hline
		
	\end{tabular}
	\caption{All estimated model parameters such as undirected edges (UE), incoming edges (IE), outgoing edges (OE), incoming weight (IW) and outgoing weight (OW) for each ward along with disease counts. All values are normalized with their corresponding maximum}
	\label{ward_count}
\end{table}
In Fig. \ref{count_edge_inc_out_w}, the normalized dengue counts, number of undirected edges of each ward, total incoming weights and total outgoing weights given in \eqref{inc_wei} and \eqref{out_wei} of each ward are shown. 
\begin{figure}[]
	\centering
	\includegraphics[scale=0.6,center]{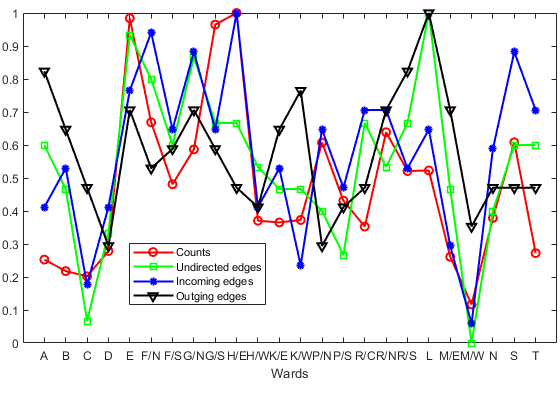}
	\caption{Normalized number of dengue counts, number of undirected edges from Fig. \ref{estimated_par}, number of incoming and outgoing edges from Fig. \ref{estimated_cas} of each ward}
	\label{count_edge_inc_out}
\end{figure}
\begin{figure}[]
	\centering
	\includegraphics[scale=0.6,center]{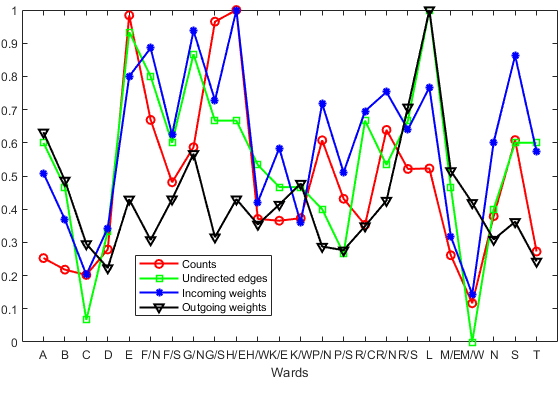}
	\caption{Normalized number of dengue counts, number of edges from Fig. \ref{estimated_par}, total incoming weights and total outgoing weights from Fig. \ref{estimated_cas} of each ward}
	\label{count_edge_inc_out_w}
\end{figure}

Several expected and some unexpected inferences are evident from  Fig. \ref{count_edge_inc_out} and \ref{count_edge_inc_out_w}:
\begin{enumerate}
	\item Dengue counts in some wards (A, E, F/N, G/N, L) are highly correlated with counts in other wards. This is evident from the high number of undirected edges of the estimated partial correlation graph incident upon these wards. While this observation provides some insight into the spread pattern, information about the direction of the spread is not available. 
	\item Some of the wards (A, R/S, L) act as sources of dengue spread. This is evident from the high number of outgoing edges of the estimated causality graph emanating from these wards.
	\item Actual dengue count can be low for these source wards (A, R/S, L). This is quite counter-intuitive and one of the key findings of this analysis. A very high daily commuter flow might explain this phenomenon.
	\item From Fig. \ref{count_edge_inc_out} and \ref{count_edge_inc_out_w}, the number of incoming edges seems to correlate well to the disease count. This seems to verify that the dengue is spread primarily by human movement. 
\end{enumerate}


\section{Conclusion}\label{sec:conclusion}
In this paper, we have investigated graphical interaction models of multivariate time series of counts using a parametric approach. The partial correlations and causalities between observed multivariate count data are defined in terms of the partial correlations and causalities between latent processes. Further a joint MLE with $ \ell_{1} $-type regularization is used to estimate the inverse spectral density matrix. To overcome the computational difficulties with the resulting mixture distributions, an MCEM algorithm with $\ell_{1}$-type regularization is proposed. Asymptotic convergence results for the sequence generated by Algorithm \ref{MCEM_algorithm} are presented and the results are verified with simulations.  Finally, the partial correlation and causality graphs are estimated for dengue count data observed weekly from each ward of Greater Mumbai city over six years. Surprisingly some wards seem to act as epicenters of disease spread even though their absolute disease counts are relatively low. Such an inference may help correct public health intervention policies in the future.

\acks{This work was partially supported by the Science and Research Engineering Board (Department of Science and Technology, Government of India). The authors thank Municipal Corporation of Greater Mumbai (MCGM) for providing the dengue dataset.}


\vskip 0.2in
\bibliography{bibdata}

\begin{thebibliography}{57}
\providecommand{\natexlab}[1]{#1}
\providecommand{\url}[1]{\texttt{#1}}
\expandafter\ifx\csname urlstyle\endcsname\relax
  \providecommand{\doi}[1]{doi: #1}\else
  \providecommand{\doi}{doi: \begingroup \urlstyle{rm}\Url}\fi

\bibitem[{Alpago} et~al.(2018){Alpago}, {Zorzi}, and {Ferrante}]{zorzi2018}
D.~{Alpago}, M.~{Zorzi}, and A.~{Ferrante}.
\newblock Identification of sparse reciprocal graphical models.
\newblock \emph{IEEE Control Systems Letters}, 2\penalty0 (4):\penalty0
  659--664, Oct 2018.
\newblock ISSN 2475-1456.

\bibitem[{Avventi} et~al.(2013){Avventi}, {Lindquist}, and
  {Wahlberg}]{ARMAgraphical}
E.~{Avventi}, A.~G. {Lindquist}, and B.~{Wahlberg}.
\newblock Arma identification of graphical models.
\newblock \emph{IEEE Transactions on Automatic Control}, 58\penalty0
  (5):\penalty0 1167--1178, May 2013.
\newblock ISSN 0018-9286.

\bibitem[Bach and Jordan(2004)]{Appgeo}
Francis~R. Bach and Michael~I. Jordan.
\newblock Learning graphical models for stationary time series.
\newblock \emph{{IEEE} Trans. Signal Processing}, 52\penalty0 (8):\penalty0
  2189--2199, 2004.

\bibitem[Brillinger(1981)]{Time-Series}
David~R. Brillinger.
\newblock \emph{Time Series Data Analysis and Theory}.
\newblock Society for Industrial and Applied Mathematics Philadelphia, 1981.

\bibitem[Brillinger(1996)]{Brillinger96remarksconcerning}
David~R. Brillinger.
\newblock Remarks concerning graphical models for time series and point
  processes.
\newblock \emph{Revista de Econometria}, 16:\penalty0 1--23, 1996.

\bibitem[Brockwell and Davis(2002)]{Brockwell}
Peter~J. Brockwell and Richard~A. Davis.
\newblock \emph{{Introduction to Time Series and Forecasting}}.
\newblock Springer, 2nd edition, mar 2002.
\newblock ISBN 0387953515.

\bibitem[Casella and Berger(2002)]{casella2002statistical}
G.~Casella and R.L. Berger.
\newblock \emph{Statistical Inference}.
\newblock Duxbury advanced series in statistics and decision sciences. Thomson
  Learning, 2002.
\newblock ISBN 9780534243128.

\bibitem[Cawley and Talbot(2010)]{overfitting}
Gavin~C. Cawley and Nicola L.~C. Talbot.
\newblock On over-fitting in model selection and subsequent selection bias in
  performance evaluation.
\newblock \emph{Journal of Machine Learning Research}, 11:\penalty0 2079--2107,
  2010.

\bibitem[Chan and Ledolter(1995)]{10.2307/2291149}
K.~S. Chan and Johannes Ledolter.
\newblock Monte carlo em estimation for time series models involving counts.
\newblock \emph{Journal of the American Statistical Association}, 90\penalty0
  (429):\penalty0 242--252, 1995.
\newblock ISSN 01621459.

\bibitem[{Ciccone} et~al.(2018){Ciccone}, {Ferrante}, and {Zorzi}]{zorzi2019}
V.~{Ciccone}, A.~{Ferrante}, and M.~{Zorzi}.
\newblock Robust identification of “sparse plus low-rank” graphical models:
  An optimization approach.
\newblock In \emph{2018 IEEE Conference on Decision and Control (CDC)}, pages
  2241--2246, Dec 2018.

\bibitem[Dahlhaus(2000)]{Graphical-Models3}
R.~Dahlhaus.
\newblock Graphical interaction models for multivariate time series.
\newblock \emph{Metrika}, \penalty0 (51):\penalty0 157--172, 2000.

\bibitem[Dahlhaus and Eichler(2003)]{DahlhausEichler}
Rainer Dahlhaus and Michael Eichler.
\newblock Causality and graphical models in time series analysis.
\newblock \emph{Oxford Stat. Sci. Ser}, 27, 01 2003.

\bibitem[David A.~Bessler(2003)]{Appstock}
Jian~Yang David A.~Bessler.
\newblock The structure of interdependence in international stock markets,
  journal of international money and finance.
\newblock \emph{Journal of International Money and Finance}, 22\penalty0
  (2):\penalty0 261--287, 2003.
\newblock ISSN 0261-5606.

\bibitem[Dempster(1972)]{Dempster72}
A.~P. Dempster.
\newblock Covariance selection.
\newblock \emph{Biometrics}, 28\penalty0 (1):\penalty0 157--175, 1972.
\newblock ISSN 0006341X, 15410420.

\bibitem[Dempster et~al.(1977)Dempster, Laird, and B.~Rubin]{emalgo}
Arthur Dempster, Natalie Laird, and Donald B.~Rubin.
\newblock Maximum likelihood from incomplete data via the em algorithm.
\newblock \emph{Journal of the Royal Statistical Society. Series B
  (Methodological)}, 39:\penalty0 1--38, 01 1977.

\bibitem[Durbin and Koopman.(2000)]{durbin2000}
J.~Durbin and S.~J. Koopman.
\newblock Time series analysis of non-gaussian observations based on
  state-space models from both classical and {Bayesian} perspectives.
\newblock \emph{Journal of the Royal Statistical Society, Ser. B}, 62:\penalty0
  3--56, 2000.

\bibitem[Eichler(2006)]{EICHLER2006}
Michael Eichler.
\newblock Fitting graphical interaction models to multivariate time series.
\newblock \emph{Proceedings of the 22nd Conference on Uncertainty in Artificial
  Intelegence}, 2006.

\bibitem[Eichler(2007)]{EICHLER200733422}
Michael Eichler.
\newblock Granger causality and path diagrams for multivariate time series.
\newblock \emph{Journal of Econometrics}, 137\penalty0 (2):\penalty0 334 --
  353, 2007.
\newblock ISSN 0304-4076.

\bibitem[Eichler(2012)]{Eichler201222}
Michael Eichler.
\newblock Graphical modelling of multivariate time series.
\newblock \emph{Probability Theory and Related Fields}, 153\penalty0
  (1):\penalty0 233--268, Jun 2012.
\newblock ISSN 1432-2064.

\bibitem[Fahrmeir and Tutz(1994)]{glm_book}
L.~Fahrmeir and G.~Tutz.
\newblock \emph{Multivariate Statistical Modelling Based on Generalized Linear
  Models}.
\newblock Springer series in statistics. Springer-Verlag, 1994.
\newblock ISBN 9780387942339.
\newblock URL \url{https://books.google.co.in/books?id=OionAQAAIAAJ}.

\bibitem[Fahrmeir and Tutz(2001)]{fahrmeir}
L.~Fahrmeir and G.~Tutz.
\newblock \emph{Multivariate Statistical Modeling Based on Generalized Linear
  Models}.
\newblock Springer-Verlag, New York, 2001.

\bibitem[Fahrmeir and Wagenpfeil(1997)]{fahrwagen}
L.~Fahrmeir and S.~Wagenpfeil.
\newblock Penalized likelihood estimation and iterative {Kalman} filtering for
  non-gaussian dynamic regression models.
\newblock \emph{Computational Statistics and Data Analysis}, 24:\penalty0
  295--320, 1997.

\bibitem[Helmut(2005)]{helmut2005new}
L{\"u}tkepohl Helmut.
\newblock \emph{New introduction to multiple time series analysis}.
\newblock Springer Berlin Heidelberg, 2005.

\bibitem[Hsiao(1982)]{HSIAO1982}
Cheng Hsiao.
\newblock Autoregressive modeling and causal ordering of economic variables.
\newblock \emph{Journal of Economic Dynamics and Control}, 4:\penalty0 243 --
  259, 1982.

\bibitem[Hue and Chiogna(2021)]{hue2021structure}
Nguyen Thi~Kim Hue and Monica Chiogna.
\newblock Structure learning of undirected graphical models for count data.
\newblock \emph{Journal of Machine Learning Research}, 22:\penalty0 50--1,
  2021.

\bibitem[{Jung} et~al.(2015){Jung}, {Hannak}, and {Goertz}]{7091904}
A.~{Jung}, G.~{Hannak}, and N.~{Goertz}.
\newblock Graphical lasso based model selection for time series.
\newblock \emph{IEEE Signal Processing Letters}, 22\penalty0 (10):\penalty0
  1781--1785, Oct 2015.
\newblock ISSN 1070-9908.

\bibitem[J.Z.~Huang and Liu(2006)]{HLP06}
M.~Pourahmadi J.Z.~Huang, N.~Liu and L.~Liu.
\newblock Covariance matrix selection and estimation via penalised normal
  likelihood.
\newblock \emph{Biometrika}, 1\penalty0 (93):\penalty0 85--98, 2006.

\bibitem[Li(1994)]{Li1994}
W.~K. Li.
\newblock Time series models based on generalized linear models: Some further
  results.
\newblock \emph{Biometrics}, 50:\penalty0 506--511, 1994.

\bibitem[Lounici(2008)]{lounici2008}
Karim Lounici.
\newblock Sup-norm convergence rate and sign concentration property of lasso
  and dantzig estimators.
\newblock \emph{Electron. J. Statist.}, 2:\penalty0 90--102, 2008.

\bibitem[McLachlan and Krishnan(2008)]{EMalgorithm}
{Geoffrey J.} McLachlan and {Thriyambakam} Krishnan.
\newblock \emph{The EM algorithm and extensions}.
\newblock Wiley series in probability and statistics. Wiley, Hoboken, NJ, 2. ed
  edition, 2008.

\bibitem[Mehdi~Jalalpour and Levin(2015)]{jalalpour}
Yulia~Gel Mehdi~Jalalpour and Scott Levin.
\newblock Forecasting demand for health services: development of a publicly
  available toolbox.
\newblock \emph{Operations Research for Health Care}, 5:\penalty0 1--9, 2015.

\bibitem[Mengersen and Tweedie(1996)]{MH_conv1}
K.~L. Mengersen and R.~L. Tweedie.
\newblock Rates of convergence of the hastings and metropolis algorithms.
\newblock \emph{The Annals of Statistics}, 24\penalty0 (1):\penalty0 101--121,
  1996.

\bibitem[Park and Park(2019)]{park2019high}
Gunwoong Park and Sion Park.
\newblock High-dimensional poisson structural equation model learning via
  $\backslash$ell\_1-regularized regression.
\newblock \emph{Journal of Machine Learning Research}, 20:\penalty0 95--1,
  2019.

\bibitem[Park and Raskutti(2015)]{park2015learning}
Gunwoong Park and Garvesh Raskutti.
\newblock Learning large-scale poisson dag models based on overdispersion
  scoring.
\newblock \emph{Advances in neural information processing systems}, 28, 2015.

\bibitem[Park and Raskutti(2017)]{park2017learning}
Gunwoong Park and Garvesh Raskutti.
\newblock Learning quadratic variance function (qvf) dag models via
  overdispersion scoring (ods).
\newblock \emph{Journal of Machine Learning Research}, 18:\penalty0 224--1,
  2017.

\bibitem[Pierce and Haugh(1977)]{pierce1977causality}
David~A Pierce and Larry~D Haugh.
\newblock Causality in temporal systems: Characterization and a survey.
\newblock \emph{Journal of econometrics}, 5\penalty0 (3):\penalty0 265--293,
  1977.

\bibitem[Qian et~al.(2015)Qian, Zhang, Zheng, Shang, Gao, and Liu]{Appbio}
Long Qian, Yi~Zhang, Li~Zheng, Yuqing Shang, Jia-Hong Gao, and Yijun Liu.
\newblock Frequency dependent topological patterns of resting-state brain
  networks.
\newblock \emph{PLOS ONE}, 10\penalty0 (4):\penalty0 1--19, 04 2015.

\bibitem[R.K.Freeland and McCabe(2004)]{freeland}
R.K.Freeland and B.P.M. McCabe.
\newblock Forecasting discrete valued low count time series.
\newblock \emph{International Journal of Forecasting}, 20:\penalty0 427--434,
  2004.

\bibitem[Roberts and Tweedie(1996)]{MH_conv3}
G.~O. Roberts and R.~L. Tweedie.
\newblock Geometric convergence and central limit theorems for multidimensional
  hastings and metropolis algorithms.
\newblock \emph{Biometrika}, 83\penalty0 (1):\penalty0 95--110, 1996.

\bibitem[Roberts and Smith(1994)]{MH_conv2}
G.O. Roberts and A.F.M. Smith.
\newblock Simple conditions for the convergence of the gibbs sampler and
  metropolis-hastings algorithms.
\newblock \emph{Stochastic Processes and their Applications}, 49\penalty0
  (2):\penalty0 207 -- 216, 1994.
\newblock ISSN 0304-4149.

\bibitem[Roy and Dunson(2020)]{roy2020nonparametric}
Arkaprava Roy and David~B Dunson.
\newblock Nonparametric graphical model for counts.
\newblock \emph{Journal of Machine Learning Research}, 21\penalty0
  (229):\penalty0 1--21, 2020.

\bibitem[Sathish et~al.(2019)Sathish, Chakraborty, and
  Mukhopadhyay]{sathish2019topology}
Vurukonda Sathish, Debraj Chakraborty, and Siuli Mukhopadhyay.
\newblock Topology selection using monte carlo expectation and maximization
  algorithm with l1-type regularization for count data.
\newblock pages 6977--6982, 2019.

\bibitem[Sims(1972)]{sims1972money}
Christopher~A Sims.
\newblock Money, income, and causality.
\newblock \emph{The American economic review}, 62\penalty0 (4):\penalty0
  540--552, 1972.

\bibitem[Songsiri(2010)]{Jitkomut}
Jitkomut Songsiri.
\newblock \emph{Graphical Models of Time Series: Parameter Estimation and
  Topology Selection}.
\newblock 2010.

\bibitem[Songsiri(2013)]{songsiri2013sparse}
Jitkomut Songsiri.
\newblock Sparse autoregressive model estimation for learning granger causality
  in time series.
\newblock In \emph{2013 IEEE International Conference on Acoustics, Speech and
  Signal Processing}, pages 3198--3202. IEEE, 2013.

\bibitem[Songsiri(2015)]{songsiri2015learning}
Jitkomut Songsiri.
\newblock Learning multiple granger graphical models via group fused lasso.
\newblock In \emph{2015 10th Asian Control Conference (ASCC)}, pages 1--6.
  IEEE, 2015.

\bibitem[Songsiri and Vandenberghe(2010)]{Songsiri:2010}
Jitkomut Songsiri and Lieven Vandenberghe.
\newblock Topology selection in graphical models of autoregressive processes.
\newblock \emph{J. Mach. Learn. Res.}, 11:\penalty0 2671--2705, December 2010.
\newblock ISSN 1532-4435.

\bibitem[Songsiri et~al.(2009)Songsiri, Dahl, and Vandenberghe]{Songsiri2009}
Jitkomut Songsiri, Joachim Dahl, and Lieven Vandenberghe.
\newblock Graphical models of autoregressive processes.
\newblock \emph{Convex Optimization in Signal Processing and Communications},
  01 2009.
\newblock \doi{10.1017/CBO9780511804458.004}.

\bibitem[Tibshirani(1996)]{LASSO}
R.~Tibshirani.
\newblock Regression shrinkage and selection via the lasso.
\newblock \emph{Journal of the Royal Statistical Society (Series B)},
  58:\penalty0 267--288, 1996.

\bibitem[TjØstheim(1981)]{Tjstheim}
Dag TjØstheim.
\newblock Granger-causality in multiple time series.
\newblock \emph{Journal of Econometrics}, 17\penalty0 (2):\penalty0 157 -- 176,
  1981.

\bibitem[Utazi et~al.(2018)Utazi, Afuecheta, and Nnanatu]{Bayesian:1}
C.~Edson Utazi, Emmanuel~O. Afuecheta, and C.~Christopher Nnanatu.
\newblock A bayesian latent process spatiotemporal regression model for areal
  count data.
\newblock \emph{Spatial and Spatio-temporal Epidemiology}, 25:\penalty0 25 --
  37, 2018.
\newblock ISSN 1877-5845.

\bibitem[W.~J.~Granger(1969)]{Granger}
Clive W.~J.~Granger.
\newblock Investigating causal relations by econometric models and
  cross-spectral methods.
\newblock \emph{Econometrica}, 37:\penalty0 424--38, 02 1969.
\newblock \doi{10.2307/1912791}.

\bibitem[Wei and Tanner(1990)]{doi:10.1080/01621459.1990.10474930}
Greg C.~G. Wei and Martin~A. Tanner.
\newblock A monte carlo implementation of the em algorithm and the poor man's
  data augmentation algorithms.
\newblock \emph{Journal of the American Statistical Association}, 85\penalty0
  (411):\penalty0 699--704, 1990.

\bibitem[Wright(1921)]{wright1921correlation}
Sewall Wright.
\newblock Correlation and causation.
\newblock \emph{J. agric. Res.}, 20:\penalty0 557--580, 1921.

\bibitem[Wright(1934)]{wright1934method}
Sewall Wright.
\newblock The method of path coefficients.
\newblock \emph{The annals of mathematical statistics}, 5\penalty0
  (3):\penalty0 161--215, 1934.

\bibitem[Zeger and Qaqish(1988)]{Zeger}
Scott~L. Zeger and Bahjat Qaqish.
\newblock Markov regression models for time series: A quasi-likelihood
  approach.
\newblock \emph{Biometrics}, 44\penalty0 (4):\penalty0 1019--1031, 1988.

\bibitem[{Zorzi} and {Sepulchre}(2016)]{zorzi2016}
M.~{Zorzi} and R.~{Sepulchre}.
\newblock Ar identification of latent-variable graphical models.
\newblock \emph{IEEE Transactions on Automatic Control}, 61\penalty0
  (9):\penalty0 2327--2340, Sep. 2016.
\newblock ISSN 0018-9286.

\end{thebibliography}

\end{document}